%% file: elsarticle-template-num.tex
\documentclass[final,5p,times,twocolumn]{elsarticle}

\usepackage{amssymb,amsthm}
\usepackage{bm,mathtools}
\theoremstyle{theorem}
\newtheorem{thm}{Theorem}
\newtheorem{prop}[thm]{Proposition}
\newtheorem{cor}[thm]{Corollary}

\newtheorem{ass}{Assumption}
\newtheorem*{cor*}{Corollary}
\newtheorem*{ass*}{Assumption}
\theoremstyle{definition}
\newtheorem*{rem*}{Remark}

\newtheoremstyle{named}{}{}{\itshape}{}{\bfseries}{.}{.5em}{\thmnote{#3}}
\theoremstyle{named}
\newtheorem*{namedtheorem}{Theorem}

\newcommand{\Vol}{\operatorname{Vol}}

\newcommand{\D}{\mathcal{D}}

\usepackage{color,xcolor,colortbl}
\definecolor{finesky}{HTML}{E6F5F0}
\definecolor{lightgray}{HTML}{ECECEC}
\definecolor{lightcyan}{rgb}{0.88,1,1}

\usepackage{caption}
\usepackage{subcaption}
\usepackage{graphicx}
\usepackage{booktabs}

\usepackage{multirow}
\usepackage{booktabs}
\usepackage{boldline,makecell}

\usepackage{hyperref}

\journal{Pattern Recognition}

\begin{document}

\begin{frontmatter}



\title{Understanding Open-Set Recognition by Jacobian Norm and Inter-Class Separation}

\author[orgy]{Jaewoo Park}
\ead{julypraise@yonsei.ac.kr}
\author[orgy]{Hojin Park}
\ead{2014142100@yonsei.ac.kr}
\author[orgw]{Eunju Jeong}
\ead{eunju.jeong@woowahan.com}
\author[orgy]{Andrew Beng Jin Teoh}
\ead{bjteoh@yonsei.ac.kr}

%
%

\address[orgy]{Electrical and Electronic Engineering Department, Yonsei University, Seoul, South Korea}
\address[orgw]{Woowa Brothers Corp., Seoul, South Korea}

%

\begin{abstract}
The findings on open-set recognition (OSR) show that models trained on classification datasets are capable of detecting unknown classes not encountered during the training process. Specifically, after training, the learned representations of known classes dissociate from the representations of the unknown class, facilitating OSR. In this paper, we investigate this emergent phenomenon by examining the relationship between the Jacobian norm of representations and the inter/intra-class learning dynamics. We provide a theoretical analysis, demonstrating that intra-class learning reduces the Jacobian norm for known class samples, while inter-class learning increases the Jacobian norm for unknown samples, even in the absence of direct exposure to any unknown sample. Overall, the discrepancy in the Jacobian norm between the known and unknown classes enables OSR. Based on this insight, which highlights the pivotal role of inter-class learning, we devise a marginal one-vs-rest (m-OvR) loss function that promotes strong inter-class separation. To further improve OSR performance, we integrate the m-OvR loss with additional strategies that maximize the Jacobian norm disparity. We present comprehensive experimental results that support our theoretical observations and demonstrate the efficacy of our proposed OSR approach.
\end{abstract}





\end{frontmatter}


\input{text.tex}


\bibliographystyle{elsarticle-num} 
\bibliography{refs}


%
%
%


\input{supplementary.tex}

\end{document}

%% file: text.tex
\section{Introduction}
\label{sec:intro}

In recent years, deep neural network (DNN) based models have demonstrated remarkable success in \textit{closed-set recognition}, where the train and test sets share the same categorical classes to classify.
In practical environments, however, a deployed model can encounter instances of
class categories \textit{unknown} during its training. Detecting
these unknown class instances is crucial in safety-critical applications such
as autonomous driving and cybersecurity. A solution to this is \textit{open-set
recognition} (OSR), where a classifier trained over $K$ known classes
can classify them and reject unknown class instances in the test
stage~\cite{hendrycks2016baseline}.

A predominant approach in DNN-based OSR is to train a discriminative model over known classes with a metric-learning loss, and derive a score (or decision) function that captures the difference between the known and unknown in terms of their representations. For the score function to work effectively, the unknown class must be dissociated from the known class in the representation space. Interestingly,
\cite{hendrycks2016baseline} along with subsequent
works~\cite{lee2018simple,liang2017enhancing,geng2020recent} observed that
training \textit{over known classes alone}  
results in this separation;
the model separates the unknown class from the known classes even though the model did not utilize any unknown class instance during its training.
 
\begin{figure}[!t]
\centering 
\includegraphics[width=.8\linewidth]{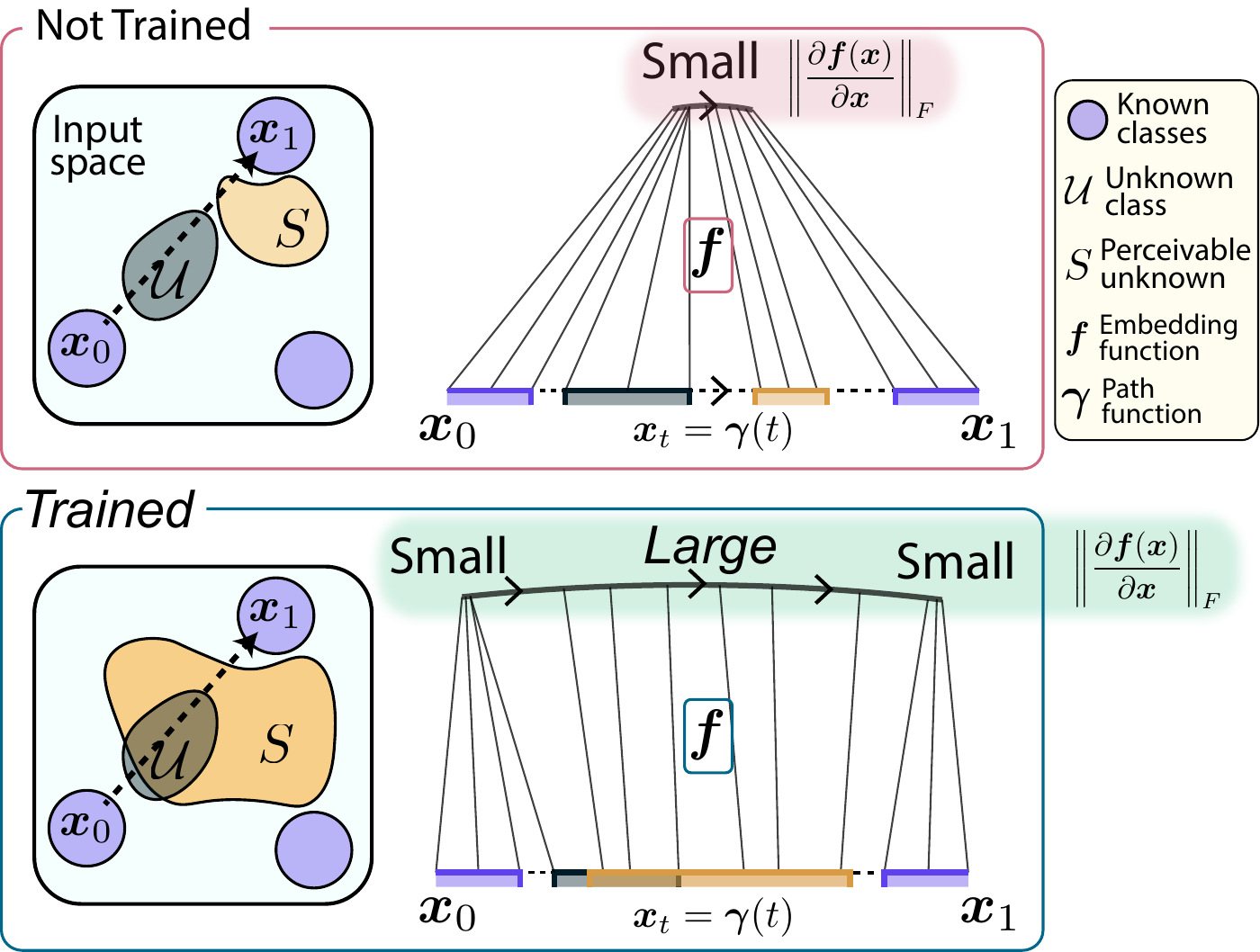}
\caption[
The main understanding of discriminative representation for open-set recognition
]{
During the closed-set metric learning, the model learns only over the known classes $C_k$, but the learning also changes the representation of \textit{unknown class}. We ask why. We discover that the intra-class learning diminishes the Jacobian norm of known class representations, while the inter-class learning increases the Jacobian norm of the unknown. The resulting disparity in Jacobian norm separates the unknown from the known.
}
\label{fig:concept}
\end{figure}

However, the underlying mechanism of this phenomenon has rarely been explored in the context of representation learning.
This work aims to analyze this phenomenon, namely, how the closed-set metric learning separates the unknown class from the known classes in the representation space.


To this end, we analyze the Jacobian norm of representation $\lVert \frac{\partial \boldsymbol{f} (\boldsymbol{x})}{\partial \boldsymbol{x}} \rVert_F$, which is the Frobenius norm of the Jacobian matrix.
We discover that inter-class separation learning within known classes plays a crucial role in OSR, as it alters the representations of \textit{unknown} class instances without direct exposure to them. Specifically, inter-class learning elevates the Jacobian norm of the unknown, whereas intra-class learning diminishes the Jacobian norm of the known. This resulting disparity between the known and unknown in terms of Jacobian norm leads to a differentiation between their respective representations.

We provide comprehensive theoretical validation for our hypothesis, which is further reinforced by a wealth of empirical evidence. Additionally, inspired by the integral role of inter-class learning in segregating unknown class instances, we develop a marginal one-vs-rest (m-OvR) loss function designed to foster substantial inter-class separation. Furthermore, we incorporate the model loss with auxiliary techniques to enhance the Jacobian norm disparity, ultimately strengthening the distinction between known and unknown classes.

The contributions of our works are summarized as follows:

\begin{enumerate}
\item We theoretically show that the closed-set metric learning separates the representations of unknown class from those of the known classes by making their Jacobian norm different. In particular, we discover that the inter-class learning is the key factor in this process as it alters the unknown class instances' representations without directly accessing them.

\item We empirically validate our theory, observing that the Jacobian norm difference between the known and unknown classes is strongly correlated to the unknown class detection performance.

\item Based on the integral role of inter-class learning for the unknown class segregation, we devise a marginal one-vs-rest (m-OvR) loss that can induce strong inter-class separation within the known classes. We further integrate the model loss with auxiliary techniques that can enhance the unknown class segregation via the Jacobian norm difference.
\end{enumerate}

We highlight that our primary objective is not to advance the state-of-the-art in the field. Rather, our foremost contribution lies in providing a theoretical elucidation of how a model gains awareness of the unknown through closed-set metric learning. Additional contributions encompass the empirical validation of our theoretical framework, as well as an examination of prevalent deep learning methodologies within the context of our proposed theory.

To the best of our knowledge, this is the inaugural study to investigate open set recognition (OSR) representations in relation to their Jacobian norm.

\section{Related Works}

\subsection{Theoretical/empirical works on OSR.}
Recent theoretical works \cite{meinke2019towards,meinke2021provably} tackle OSR with theoretical guarantees on the performance but with specific distributional modeling assumptions (e.g.~Gaussian mixture). \cite{fang2021learning} conduct theoretical studies in a more general setting by extending the classical closed-set PAC framework \cite{valiant1984theory} to open-set environments, deriving analytical bounds of the generalization error in the context of OSR. \cite{liu2018open} relates OSR to transfer learning and interprets the unknown class samples as covariate shifts. This enables the substitution of theoretical bounds derived in the transfer learning setting \cite{ben2007analysis} to open-set environments.
On the other hand, \cite{dhamija2018reducing} observes that for a model trained only with known class samples, the magnitudes of representation vectors tend to exhibit relatively larger values over the known class than over the unknown ones. 
\cite{dietterich2022familiarity} empirically proved that the standard discriminative models detect unknown classes mainly based on their unfamiliar features rather than based on the novelty of unknown category.

\subsection{OSR Methods}
For a general, broad survey on OSR models, the readers are recommended to \cite{geng2020recent,yang2021generalized}. Here, we focus on reviewing state-of-the-art OSR models, mainly focusing on discriminative ones. 

The basic baseline model \cite{neal2018open,hendrycks2016baseline,DBLP:journals/corr/abs-2110-06207} trained by softmax cross-entropy loss is known to perform both closed-set classification and unknown class detection reasonably effectively.
To enhance its unknown detection mechanism, 
OpenMax \cite{bendale2016towards} applied probabilistic modification on the softmax activation based on extreme value theory. DOC \cite{shu2017doc} replaced the softmax cross entropy with the one-vs-rest logistic regression, finding its effectiveness on invalid topic rejection in natural language.
RPL \cite{chen2020learning} proposed to maximize inter-class separation in the form of reciprocal, followed by a variant \cite{chen2021adversarial} that utilizes synthetic, adversarially generated unknown class. CPN \cite{yang2020convolutional} learns embedding metrics 
by modeling each known class as a group of multiple prototypes.
PROSER \cite{zhou2021learning} leverages latent mixup samples \cite{zhang2017mixup,verma2019manifold} as a generated unknown class and places their representations near the known class representations. 
\cite{jang2022collective}, on the other hand, proposed a collection of multiple one-vs-rest networks to mitigate the over-confidence and poor generalization issue, and utilizes a collective decision score for effective OSR. 

Recently, \cite{DBLP:journals/corr/abs-2110-06207} demonstrated that the basic SCE baseline could outperform all other OSR baselines if the SCE model is trained with strong data augmentation and utilizes state-of-the-art optimization techniques. On the other hand, \cite{kasarla2022maximum} showed that a prior on well separated discriminative embedding is still critical for effective open-set recognition.

\subsection{Jacobian Norm in Deep Discriminative Models}
Within the domain of discriminative learning, though not explicitly in the context of OSR, the Jacobian of the representation function has been examined in closed-set settings within various contexts. \cite{sokolic2017robust,novak2018sensitivity} demonstrate that the explicit minimization of the Frobenius norm of the Jacobian of classification prediction output, specifically the softmax output and logit, promotes a smoothness prior on it, subsequently enhancing the generalization of recognition in closed-set scenarios.
Nevertheless, the explicit computation of the Jacobian demands substantial computational resources. To address this, \cite{varga2017gradient} has introduced an efficient method of computing the Jacobian norm via its random projection, serving as an unbiased estimator of the raw Jacobian norm.

\cite{novak2018sensitivity,jakubovitz2018improving,hoffman2019robust} have noted that
the smoothness prior, as enforced by Jacobian norm penalization, reduces the sensitivity of the network output to minute input perturbations, thereby making the network robust against adversarial examples.
On a theoretical level, \cite{zhang2018three} has identified a close connection between weight decay and the Jacobian norm, establishing that under ideal conditions, a gradient update with weight decay equates to penalizing the Frobenius norm of the Jacobian matrix of representation.

However, all preceding studies on Jacobian analysis have been confined to the context of closed-set learning. To the best of our knowledge, our research represents the first instance of analyzing the Jacobian norm within an open-set scenario, providing a rigorous examination of its relationship to the unknown class.

\section{Theory: Understanding the Separation of Unknown Class Representations via Jacobian Norm}
\label{sec:theory}

We theoretically demonstrate that training a discriminative model over known classes separate the representations of known classes from those of the unknown class by decreasing Jacobian norm over the known classes while increasing the Jacobian norm over the unknown class (Cor. \ref{cor:sep} in Sec. \ref{sec:theory_derivation}). The limitation of Jacobian norm theory is given in Sec. \ref{sec:theory_limit}. Our observation is summarized in Sec. \ref{sec:theory_summary} with its depiction in Fig. \ref{fig:summary_theory}.

\subsection{Problem Setup and Notation} 
During closed-set metric learning, 
the representation embedding function $\boldsymbol{f}: \mathbb{R}^d \to \mathbb{R}^{d_z}$ of a discriminative model is trained to minimize \textit{intra-class distances} $\D(\boldsymbol{f}(\boldsymbol{x}), \boldsymbol{w}_{y})$ and maximize \textit{inter-class distances} $\D(\boldsymbol{f}(\boldsymbol{x}), \boldsymbol{f}(\boldsymbol{x}'))$ 
for known class samples $\boldsymbol{x}$ and $\boldsymbol{x}'$ paired with different class labels $y$ and $y'$ ($y \neq y'$). The prototype vector $\boldsymbol{w}_y {\in} \mathbb{R}^{d_z}$ is a proxy  for the $y$-th known class $C_y$, and is formulated as a learnable parameter. The known set $\mathcal{K} {=} \cup_{k=1}^K C_k$ consists of $K$ disjoint disconnected known classes $C_k$. The train samples $\boldsymbol{x}$ and $\boldsymbol{x}'$ are sampled from the known set, while the labels $y$ and $y'$ from the corresponding label space $\mathcal{Y}_\mathcal{K} = \{1,\dots, K\}$. The open set $\mathcal{O} :=\mathcal{X} \setminus \mathcal{K}$ is the complement of the known set in the (bounded) global space $\mathcal{X} = [-1,1]^d \subseteq \mathbb{R}^{d}$. The unknown class $\mathcal{U}$ we consider is a \textit{proper} subset of the open set $\mathcal{O}$. Since our task is not to discriminate within the unknown class, we treat the unknown class as a single class, although it may consist of a diverse type of object. 

During training, the model has no access to the unknown class $\mathcal{U}$, and is trained only with the $K$ number of known classes to discriminate them. After training, the OSR model should
not only discriminate each class in the known set but also need to differentiate the unknown from the known. Hence, the unknown class should be separated from all known classes in the representations space such that $\boldsymbol{f}(\mathcal{K}) \cap \boldsymbol{f}(\mathcal{K}) = \varnothing$.

\subsection{Derivation of the Theory}
\label{sec:theory_derivation}

We prove our theoretical claims by observing how the embedding function $\boldsymbol{f}$ changes on a class interpolating path (i.e., a path 
$\boldsymbol{\gamma} {:} [0,1] {\to} \mathcal{X}$ that interpolates two different known classes $C_i$ and $C_j$ by traversing $t$ from $0$ to $1$ with $\boldsymbol{x}_0 {\in} C_i$ and $\boldsymbol{x}_1 {\in} C_j$ as depicted in Fig.~\ref{fig:concept}). The detailed assumptions and full proofs to the theoretical statements are given in  \ref{sec:assumption_theory} and \ref{sec:proof_theory}, respectively.

Firstly, we show that,
during the closed-set supervision, the intra-class distance minimization minimizes the length of the projected path over the known class:

\begin{prop}
\label{prop:collapse}
Minimizing intra-class distances $\D(\boldsymbol{f}(\boldsymbol{x}), \boldsymbol{w}_k)$ to $0$ for all $\boldsymbol{x} \in C_k$ minimizes the length of the projected path $\boldsymbol{f}(\boldsymbol{\gamma}([0,1]) \cap C_k)$ for an arbitrary path $\gamma$ from $C_k$.
\end{prop}

On the other hand, the inter-class distance maximization is presumed to increase the length of any linear path between the known classes $C_i$ and $C_j$ in the representation space by Assumption \ref{ass:regular_embedding}b. In summary, intra-class distance minimization reduces the projected path length, while the inter-class distance maximization increases the projected path length.

Now, the increasing/decreasing trend of the projected path length due to the metric learning is transferred to the Jacobian norm $\lVert \frac{d \boldsymbol{f} (\boldsymbol{\gamma}(t))}{d t} \rVert_2$ via the path length equation 
\begin{equation}
\text{length}(\boldsymbol{f} {\circ} \boldsymbol{\gamma}) = \int_0^1 \left\lVert \frac{d \boldsymbol{f}(\boldsymbol{\gamma}(t))}{dt} \right\rVert_2 dt.
\end{equation}
Accordingly, we expect that intra-class distance minimization minimizes the Jacobian norm over the known class intersecting path. In contrast, inter-class distance maximization increases the Jacobian norm over the open set intersecting path. This description, however, is constrained to the local paths.
The following theorem assures that this phenomenon is extendible from the local path to the global region. In other words, the closed-set metric learning minimizes the Jacobian norm over the known classes and increases the Jacobian norm over the open set $\mathcal{O}$.

\begin{thm}
Let $C_i$, $C_j$, and $C_k$ be different known classes.
\label{thm:grad_norm}
\begin{enumerate}
\item[(a)]
Minimizing intra-class distances $\D(\boldsymbol{f}(\boldsymbol{x}), \boldsymbol{w}_k)$ for all $\boldsymbol{x} \in C_k$ minimizes $\lVert \frac{\partial \boldsymbol{f} (\boldsymbol{x})}{\partial \boldsymbol{x}} \rVert_F$ over $C_k$.
\item[(b)]
Maximizing inter-class distances $\D(\boldsymbol{f}(\boldsymbol{x}), \boldsymbol{f}(\boldsymbol{x}'))$ for all $\boldsymbol{x} \in C_i$ and $\boldsymbol{x}' \in C_j$ strictly increases $\int_{\mathcal{O}} \lVert \frac{\partial \boldsymbol{f} (\boldsymbol{x})}{\partial \boldsymbol{x}} \rVert_F \; d\boldsymbol{x}$.
\end{enumerate}
\end{thm}

Theorem \ref{thm:grad_norm}b indicates that the length of the projected path can be accessed from the global integral of the Jacobian norm. Thereby, we find that the strictly increasing trend of Jacobian norm integral is positively correlated to the strictly increasing trend of the projected inter-class path length. Based on our overall observations, we deduct the below corollaries:

\begin{cor}
\label{cor:intra}
Minimizing the intra-class distances minimizes the Jacobian norm $\lVert \frac{\partial \boldsymbol{f} (\boldsymbol{x})}{\partial \boldsymbol{x}} \rVert_F$ over the known classes $\mathcal{K}$.
\end{cor}

\begin{cor}
\label{cor:support}
Maximizing the inter-class distances strictly increases strictly increases 
\begin{equation}
\Vol(S) \text{ and/or }
\mathbb{E}_{\boldsymbol{x} \sim S} [
\lVert
\tfrac{\partial \boldsymbol{f}}{\partial \boldsymbol{x}} 
\rVert_F
]
\end{equation}
where $S$ is the support of Jacobian norm 
\begin{equation}
S := \{ \boldsymbol{x} \in \mathcal{O} : \lVert \tfrac{\partial \boldsymbol{f}(\boldsymbol{x})}{\partial \boldsymbol{x}} \rVert_F >0 \},
\end{equation}
whose Jacobian norm is greater than $0$, and $\Vol(S)$ is the volume of $S$.
Hence, if $S \cap \mathcal{U} \neq \varnothing$, then the inter-class maximization enlarges the volume $\Vol(\mathcal{U} \cap S)$ and/or  increases the Jacobian norm of unknown class samples $\boldsymbol{x} \in \mathcal{U} \cap S$.
\end{cor}

Hence, maximizing the inter-class distances between the known classes access to the unknown class samples indirectly via the region $S$ of high Jacobian norm, and increases the Jacobian norm of unknown class representations.

Overall, by metric learning, the model increases the expected Jacobian norm difference between the known and unknown
\begin{equation}
\label{eq:gnd}
\underset{\boldsymbol{x} \sim \mathcal{U}}{\mathbb{E}} [
\lVert
\tfrac{\partial \boldsymbol{f} (\boldsymbol{x})}{\partial \boldsymbol{x}} 
\rVert_F
]
-
\underset{\boldsymbol{x} \sim \mathcal{K}}{\mathbb{E}} [
\lVert
\tfrac{\partial \boldsymbol{f} (\boldsymbol{x})}{\partial \boldsymbol{x}} 
\rVert_F
].
\end{equation}
The increased \textit{Jacobian norm difference} then separates the known classes from  the unknown class in the representation space:

\begin{cor}
\label{cor:sep}
The inter/intra-class learning separates the unknown class from known classes in the representation space by inducing the Jacobian norm difference between the known and unknown.
\end{cor}


\subsection{Limitation of the Theory on Jacobian Norm}
\label{sec:theory_limit}

We highlight that the Jacobian norm characteristic \textbf{is only one of the many explanatory factors} that demystifies how closed-set metric learning derives OSR; our analysis does not fully characterize all connections between closed-set metric learning and OSR. One apparent phenomenon our theory does not explain is that known and unknown representations can be separated in the metric space with having the same Jacobian norm value. Moreover, our theory is limited in characterizing the support set $S$. As the support set does not include the whole part of the open set, there would be some unknown class that is not included in the support. In this case, the Jacobian norm difference indicated in Eq. \eqref{eq:gnd} would not be explanatory.

\begin{figure*}[t]
\centering 
\includegraphics[width=.75\linewidth]{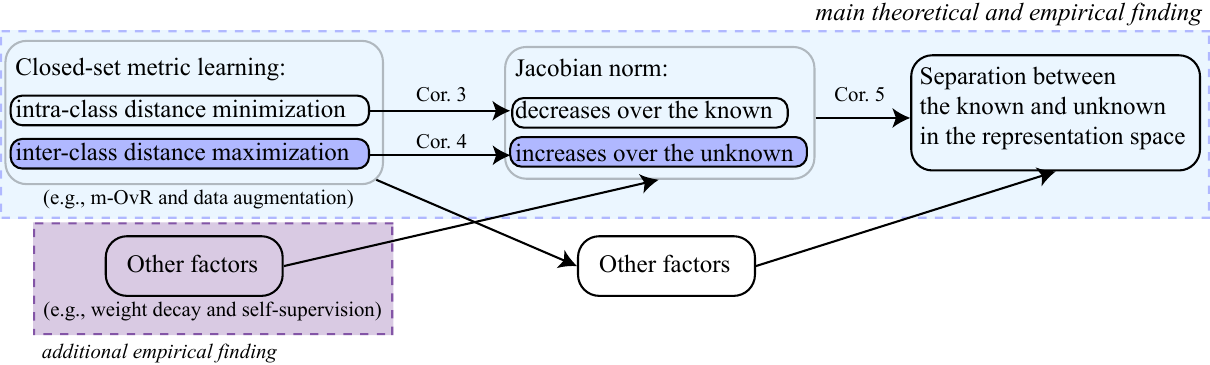}
\caption{
{ 
The summary of our theory on how a model becomes aware of the unknown by the closed-set metric learning over the known classes.
}
}
\label{fig:summary_theory}
\end{figure*}

\subsection{Summary of Theory}
\label{sec:theory_summary}

Training a discriminative model over the known classes reduces the Jacobian norm over known class samples, while increasing the volume of region of high Jacobian norm in the open set. Due to the increased volume of high Jcaobian norm region, the unknown class samples likely fall into this region, and thus involve high Jacobian norm values. Overall, the embedding representations of known classes are separated from those of unknown class because the Jacobian norms of known classes are low while the Jacobian norms of unknown class are high.
Our theoretical finding is summarized in Fig.~\ref{fig:summary_theory}.


\section{Empirical Verification of the Theory}
\label{sec:exproof}

In this section, we empirically verify the theory developed in Sec.~\ref{sec:theory} in multiple aspects.

\subsection{Experiment Setup}

We empirically analyze the relationship between the Jacobian norm difference and the unknown class detection to evidence our theoretical analysis. To this end, we train our proposed model as described in Sec.~\ref{sec:model} and \ref{sec:expcomp}, and evaluate  over the standard OSR benchmark datasets \cite{neal2018open}. To compute the degree of separation between known and unknown, we use the detection score provided in Sec.~\ref{sec:model_score} and evaluate the area under the receiver-operating-characteristic curve (AUC) metric \cite{bradley1997use}. The discriminative (cluster) quality of known class representations is measured in Davies-Bouldin Index (DBI) \cite{davies1979cluster}, which measures the ratio of intra-class distance to inter-class distance. 
All experiments are conducted with one 12GB GPU RTX 2080-ti. Due to resource limitations, empirical observations are made on standard OSR datasets rather than recently proposed high-resolution OSR datasets \cite{DBLP:journals/corr/abs-2110-06207}.

\noindent
\textbf{Datasets.} For the empirical analysis, we test on the standard OSR datasets as described in Protocol A of Sec. \ref{sec:exp_data}. Each dataset consists the $K$ number of known classes and $1$ unknown class, overall constituting $K+1$ semantic classes. The unknown class can be constituted by a diverse set of semantic classes, but is regarded as a single chunk. The known classes must have no semantic overlap with the unknown class.

\subsection{Empirical Observations}


\begin{figure*}[t]
\centering 
\includegraphics[width=.85\linewidth]{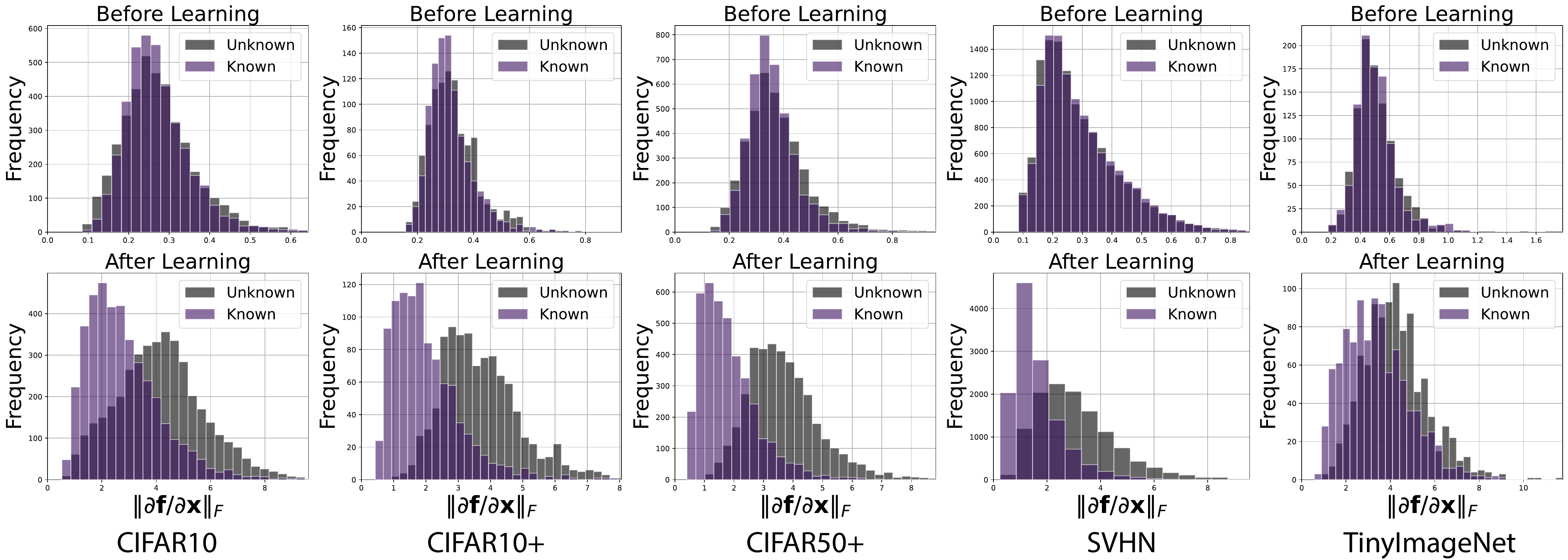}
\caption{
The distribution of Jacobian norms of representations \textbf{before and after training}. Although the model is trained only on the known class data, the model learns to increase the Jacobian norm of unknown class representation, while lowering the Jacobian norm of the known class representation.
}
\label{fig:before_after}
\end{figure*}

\begin{figure}[t]
\centering 
\includegraphics[width=.85\linewidth]{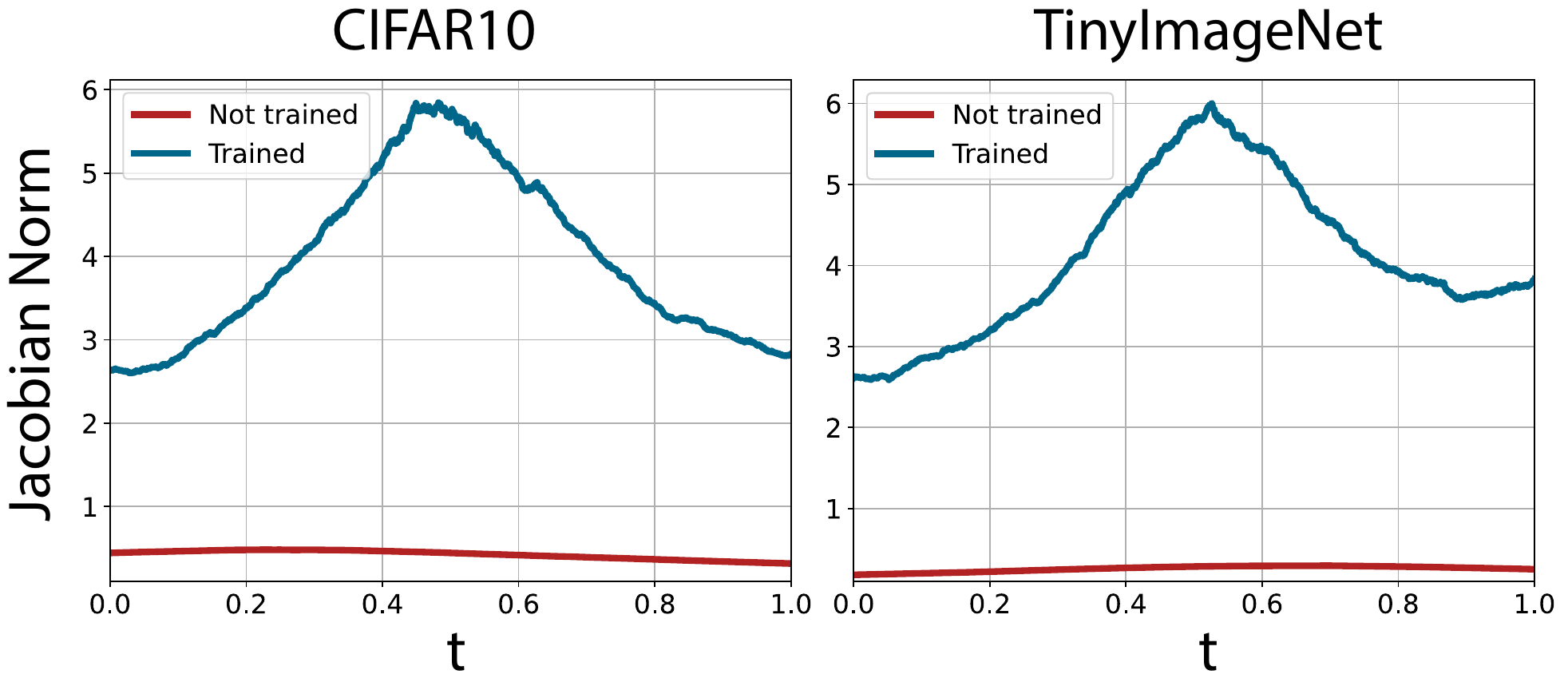}
\caption{
Given known class samples $\boldsymbol{x}_0 \in C_i$ and $\boldsymbol{x}_1 \in C_2$ from two different classes $C_i$ and $C_j$, we \textbf{linearly interpolate} between $\boldsymbol{x}_0$ and $\boldsymbol{x}_1$  by $\boldsymbol{x}_t := (1-t)\boldsymbol{x}_0 + t\boldsymbol{x}_1$. Then, we measure the Jacobian norm of the representation $\boldsymbol{f}(\boldsymbol{x}_t)$. When $t \approx 0.5$, the interpolated sample $x_t$ passes through the open set, where unknown class samples arise.
}
\label{fig:interpolate}
\end{figure}

\noindent
\textbf{Jacobian norm before and after training.}
Fig.~\ref{fig:before_after} demonstrate that the gradient norm separates the representations only after training. Fig.~\ref{fig:interpolate} displays the gradient norm over the linearly interpolated data samples $\boldsymbol{x}_t$ for $t \in [0,1]$ between two different class samples $\boldsymbol{x}_0 \in C_i$ and $\boldsymbol{x}_1 \in C_j$.
It shows that the interpolated samples inside the open region have a larger gradient norm than those in the known classes. These empirical observations support our theory.

In practice, however, the inter/intra-class distance optimizations conflict; thus, the overall gradient norm increases for both the known and unknown. 

Moreover, on some datasets (SVHN and TinyImageNet), the inter-class separation may not be substantial due to innate data characteristics such as small inter-class data variance. Accordingly, based on Theorem \ref{thm:grad_norm}b, the weak inter-class separation induces relatively smaller difference in Jacobian norm between the known and unknown, resulting a larger overlap between them.

\begin{figure*}[t]
\begin{center}
\includegraphics[width=.85\linewidth]{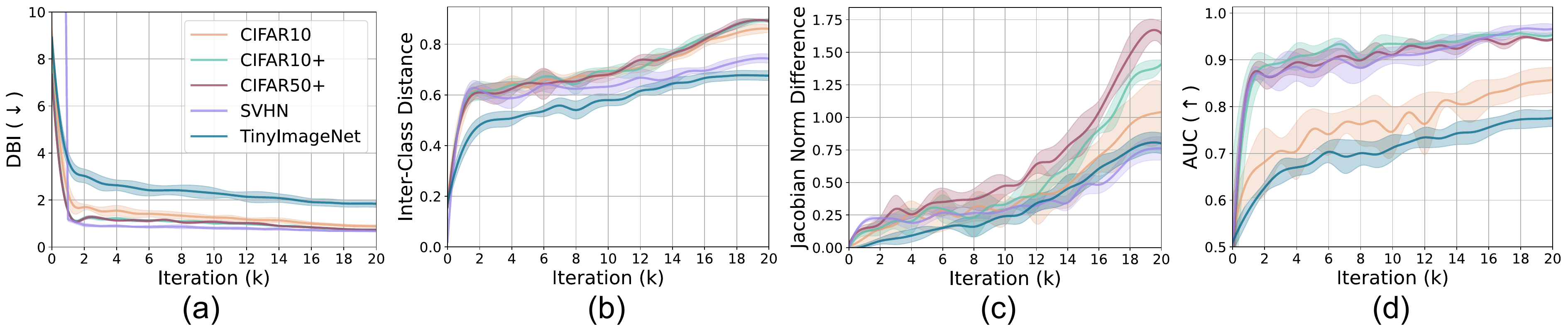}
\end{center}
\caption{
Several metrics are measured while a discriminative model (ours) is trained. (a) The discriminative quality of known class representations is measured in DBI. (b) The averaged inter-class distances between known classes. (c) The Jacobian norm difference between the known and unknown classes. (d) The degree of separation between known and unknown class representations. \textbf{All metrics are improved as the discriminative model learns}.
}
\label{fig:iter_vs_metric_ours}
\end{figure*}

\noindent
\textbf{The dynamics of the Jacobian norm during training.}
Fig.~\ref{fig:iter_vs_metric_ours} shows the dynamics of different quantities during training. The intra/inter-class distance optimization increases the quality of cluster separation measured by DBI.
Accordingly, the linear projected path length between different known classes in the representation space increases (Fig.~\ref{fig:iter_vs_metric_ours}b). As a result, the model increases both the Jacobian norm difference (Fig.~\ref{fig:iter_vs_metric_ours}c) and the degree of separation between known and unknown classes (Fig.~\ref{fig:iter_vs_metric_ours}d) as claimed by the theory. 

Although the global trend has a simple correspondence between these metrics, a more careful look at the graphs of Fig.~\ref{fig:iter_vs_metric_ours} shows that the metrics involve different phases during training. Specifically, the intra/inter ratio is stable at the early stage of training. On the other hand, the inter-class distance is still increasing even at a later stage. The Jacobian norm difference rises more gradually, and the rate of increase becomes large at the last stage. The separation between the known and unknown also increases largely at the early stage but continues to improve even later in training. These observations show that the known and unknown class representations are separated as the model makes their Jacobian norm different. Still, the Jacobian norm is not the only factor contributing to their separation. 

\begin{figure*}[t]
\centering 
\includegraphics[width=.85\linewidth]{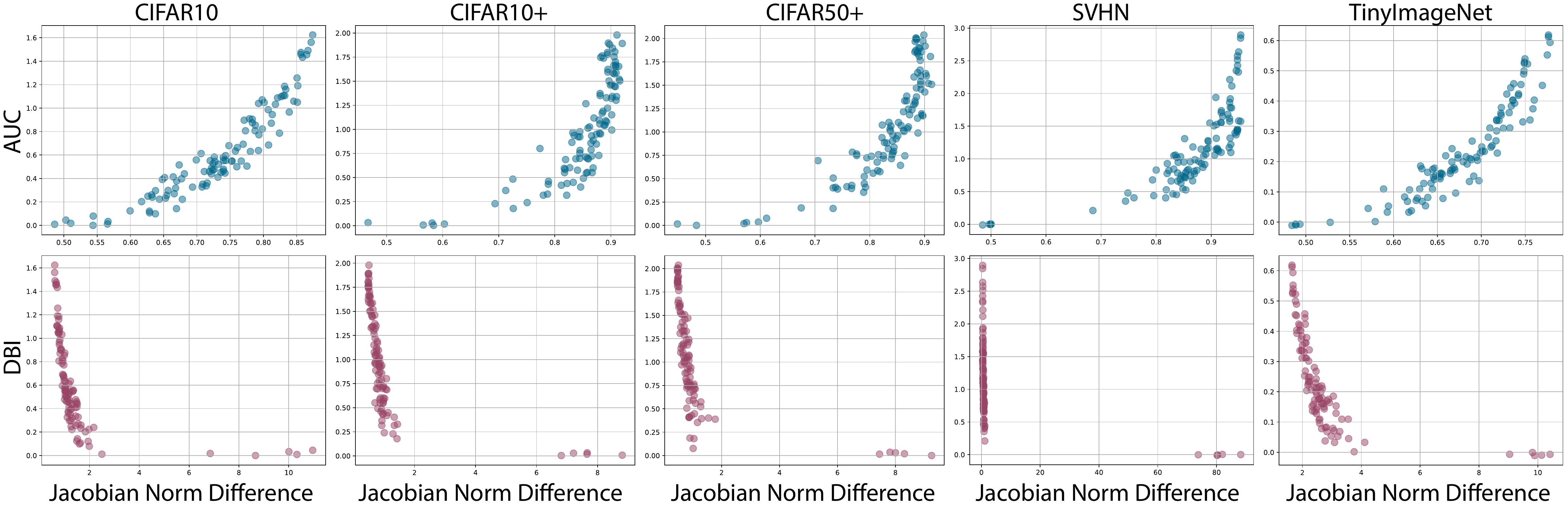}
\caption{
We measure the detection performance (in AUC), the discriminative quality of known classes (in DBI), and the averaged Jacobian norm difference \textbf{for a single model during different training iterations, indicating a strong correlation between these metrics}.
}
\label{fig:scatter_all_fixed_model}
\end{figure*}

\noindent
\textbf{The correlation between Jacobian norm and discriminative metrics.}
For each dataset, we measure the following three metrics during different training iterations: the discriminatory quality of known class representations (DBI), the unknown class detection performance (AUC), and the averaged Jacobian norm difference between known and unknown classes.

Fig.~\ref{fig:scatter_all_fixed_model}(1st row) shows that the degree of separation between the known and unknown strongly correlates to the Jacobian norm difference. This observation evidences our theoretical claim that the closed-set metric learning separates the unknown by increasing their Jacobian norm difference during training. There is, however, nonlinearity between these two metrics, showing that the Jacobian norm difference is not the only factor contributing to the separation of unknown class representation.

Fig.~\ref{fig:scatter_all_fixed_model}(2nd row) shows a similar correlation trend between the intra/inter-class distance ratio (DBI) and the Jacobian norm difference. However, the nonlinearity between them is severe. The plot indicates that the Jacobian norm difference abruptly increases at a later stage of training where the intra/inter ratio is already small and stable.

\begin{figure*}[t]
\centering 
\includegraphics[width=.85\linewidth]{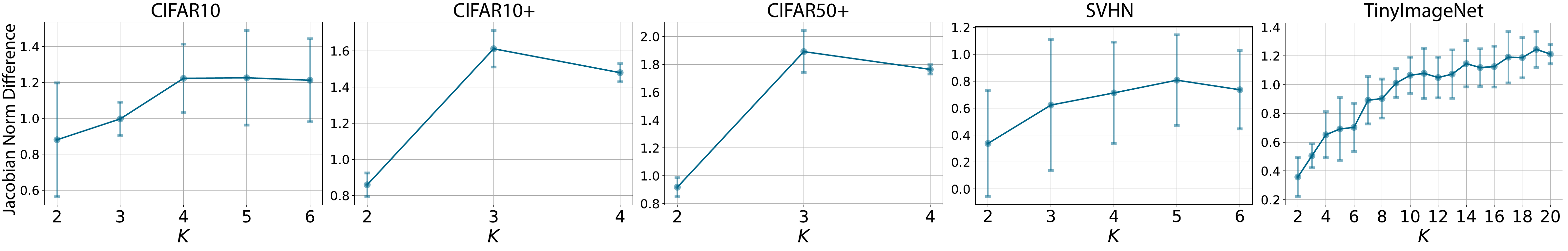}
\caption{
Increasing the number $K$ of known classes increases the Jacobian norm difference between the known and unknown classes.
}
\label{fig:fig_k_jnd}
\end{figure*}

\noindent
\textbf{The relation between the Jacobian norm difference and the number of discriminative classes.}
Theorem \ref{thm:grad_norm} states that the inter-class distance maximization between a single pair of inter classes $(C_i, C_j)$ can cause to increase in the Jacobian norm difference. Therefore, we hypothesize that a larger number of inter-class pairs would improve the Jacobian norm difference, contributing to better separation between known and unknown class representations. The results are given in Fig.~\ref{fig:fig_k_jnd} supports the hypothesis by showing that the Jacobian norm difference tends to become larger with a larger number $K$ of known classes. We note that the exceptions may occur as some known classes are more similar to the unknown class examples; adding to the train data a known class that is similar to the unknown may slightly reduce the Jacobian norm difference.

\begin{figure*}[!t]
\begin{center}
\includegraphics[width=.85\linewidth]{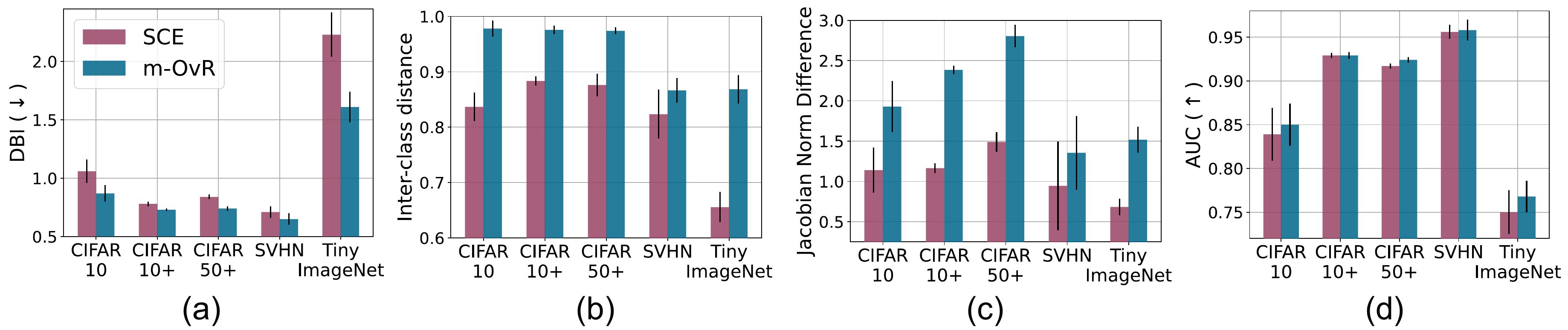}
\end{center}
\caption{
Comparison of the loss functions (SCE and m-OvR) with respect to (a) the discriminative quality in DBI, (b) the average of the pairwise distance between class-wise mean features, (c) Jacobian norm difference, and (d) unknown class detection performance in AUC.
}
\label{fig:comp_loss}
\end{figure*}

\begin{figure*}[t]
\centering 
\includegraphics[width=.85\linewidth]{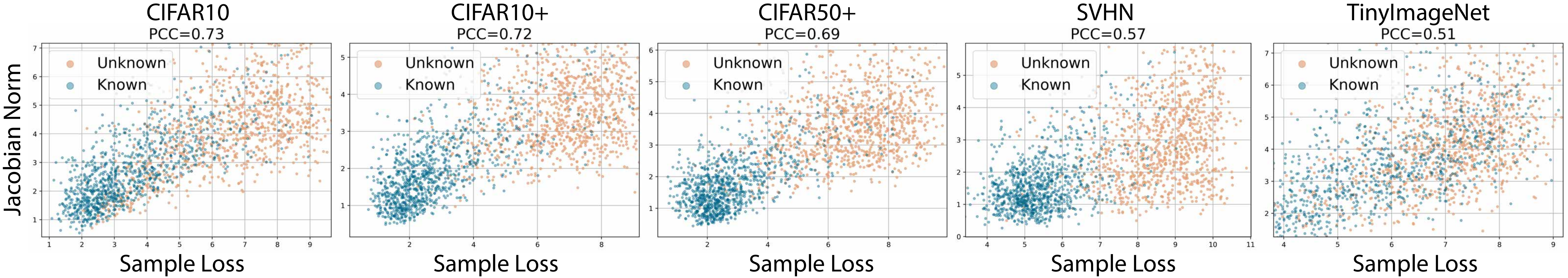}
\caption{
The correlation between sample-wise losses and the Jacobian norm of the corresponding representations.
}
\label{fig:scatter_sample_loss_vs_jn}
\end{figure*}

\section{Method}
\label{sec:model}

We develop an effective OSR method based on our theoretical finding given in Fig.~\ref{fig:summary_theory}. Firstly, we devise a margin-based one-vs-rest that can induce powerful inter-class separation between different known classes. Then, we integrate the loss term with other regularizers that enhance the separation of the unknown via the Jacobian norm difference. Finally, for the unknown class detection in the inference stage, we utilize the sample-wise loss function as it is aware of both the Jacobian norm difference and proximity to the known class prototypes.


\subsection{Training: marginal One-vs-Rest Loss (m-OvR)}
Our analysis indicates that the powerful inter-class separation is the key to separate the known from the unknown in the Jacobian norm and therefore in the representation space. Motivated upon this theory, we devise a marginal one-vs-rest (m-OvR) loss that induces powerful inter-class separation by preventing the collapse between inter-class prototypes $\boldsymbol{w}_k$ and effective inter-class gradients.
The m-OvR loss is given by
\begin{equation}
\label{eq:ovr_loss}
\mathcal{L}(\boldsymbol{x}, y) = -\sum_{k=1}^K 
1\{y=k\} \log p(k|\boldsymbol{x}) \\
+ 1\{y=k\} \log \left( 1 - p(k|\boldsymbol{x}) \right)
\end{equation}
where $(\boldsymbol{x},y)$ is a labeled sample, and $1\{\cdot\}$ is an indicator function. The class probability $p(k|\boldsymbol{x})$ is given by $\sigma(Ts_k)$ where $\sigma$ is the sigmoid activation, $s_k$ is the cosine similarity between the representation $\boldsymbol{f}(\boldsymbol{x})$ and the $k$-th class proxy prototype $\boldsymbol{w}_k$, and $T$ is a scale term to calibrate the sigmoid probability.

During training, the bare minimization of the loss in Eq.~\eqref{eq:ovr_loss} involves a harmful behavior; particularly, minimizing the loss in Eq.~\eqref{eq:ovr_loss} collapses the inter-class prototypes as observed by below proposition:

\begin{prop}
\label{prop:ovr_collapse}
The minimum OvR loss collapses all prototypes $\boldsymbol{w}_k = \boldsymbol{w}_{k'}$ except $\boldsymbol{w}_y$. 
\end{prop}


This inter-class collapse weakens the inter-class separation. 
We mitigate this situation by inserting a margin in the similarity computation; namely, during the training of the OvR metric-learning loss, the similarity is computed by
\begin{equation}
s_k = \cos( \arccos( \boldsymbol{w}_k \cdot \boldsymbol{f}(\boldsymbol{x})) + m)
\end{equation}
where $m>0$ is the margin. The margin ensures an angular gap of degree $2m$ between inter-class prototypes, thus preventing their collapse:
\begin{prop}
\label{prop:ovr_margin}
For the nonzero margin $m >0$, however, the angle gap can be assured between different prototypes $\measuredangle(\boldsymbol{w}_{k_1}, \boldsymbol{w}_{k_2}) {\geq} 2m $.
\end{prop}


In addition, the proposed m-OvR induces more powerful inter-class separation than the standard softmax cross-entropy (SCE) loss:

\begin{prop}
\label{prop:movr_sce}
Assume $s_y > 0$. Then, the inter-class gradient for the m-OvR $\frac{\partial s^{\text{m-OvR}}_k}{\partial \theta}$ is greater than that for the SCE $\frac{\partial s^{\text{SCE}}_k}{\partial \theta}$.
\end{prop}

Therefore, m-OvR is more effective at increasing the Jacobian norm difference and, hence, the unknown class detection performance accordingly. 

The empirical observations given in Fig.~\ref{fig:comp_loss} indicate the effectiveness of m-OvR compared to the SCE loss in terms of Jacobian norm difference, discriminative quality of known class representations, and the unknown class detection performance based on the detector in Sec.~\ref{sec:model_score}.

\begin{figure*}[t]
\begin{center}
\includegraphics[width=.65\linewidth]{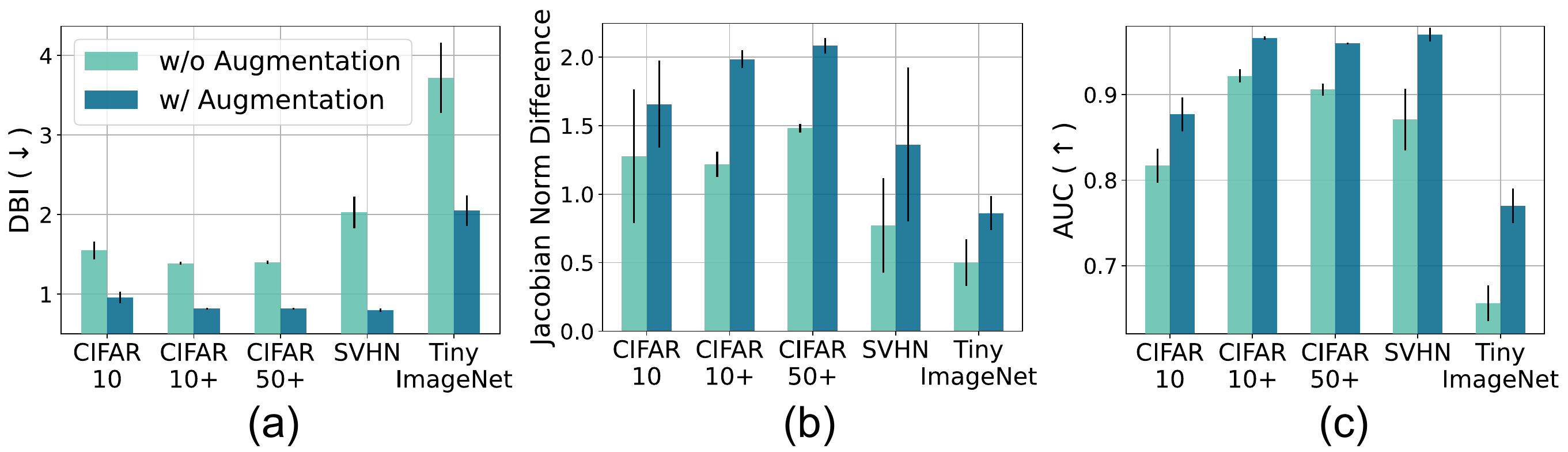}
\end{center}
\caption{
The effect of data augmentation with respect to (a) the discriminative quality in DBI, (b) Jacobian norm difference, and (c) unknown class detection performance in AUC.
}
\label{fig:comp_aug}
\end{figure*}

\begin{figure*}[t]
\begin{center}
\includegraphics[width=.65\linewidth]{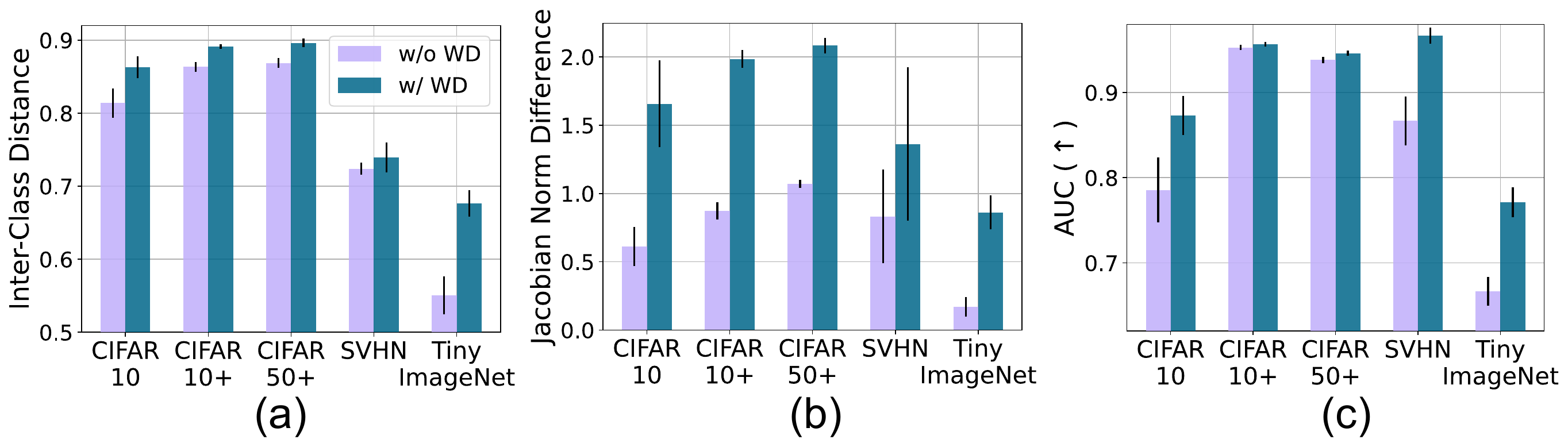}
\end{center}
\caption{
The effect of weight decay (WD) with respect to (a) averaged inter-class distance, (b) Jacobian norm difference, and (c) unknown class detection performance in AUC.
}
\label{fig:comp_wd}
\end{figure*}

\begin{figure*}[t]
\begin{center}
\includegraphics[width=.65\linewidth]{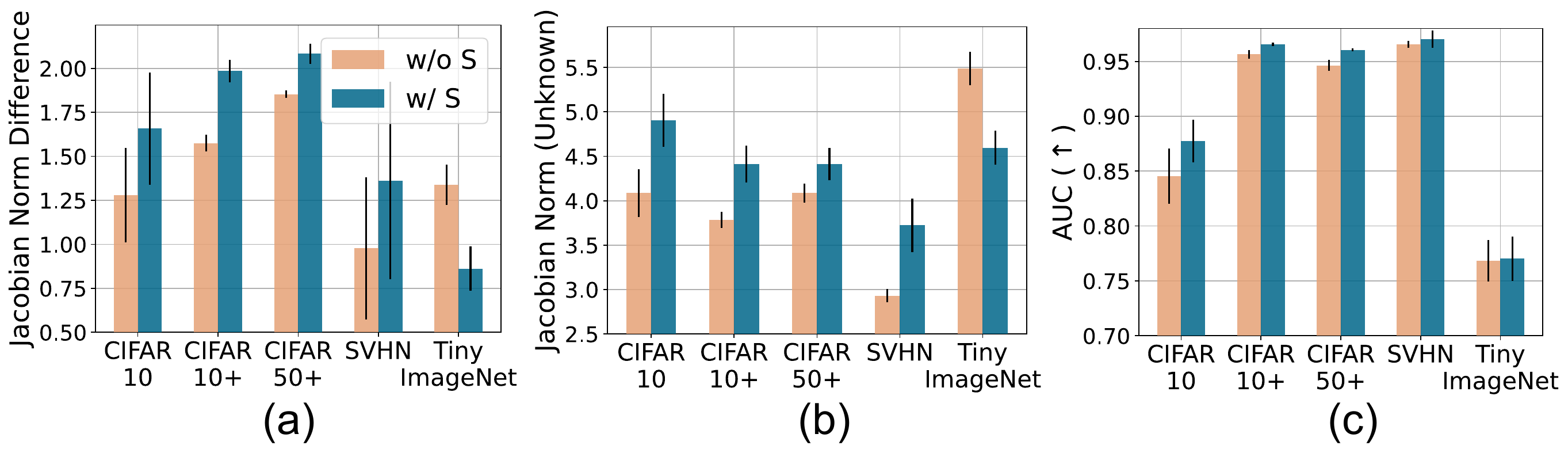}
\end{center}
\caption{
The effect of auxiliary self-supervision (S) with respect to (a) Jacobian norm difference, (b) Jacobian norm of unknown class, and (c) unknown class detection performance in AUC.
}
\label{fig:comp_self}
\end{figure*}


\subsection{Training: Subsidiary Techniques to Improve OSR}
\label{sec:model_sub}
Using the Jacobian norm principle from Sec.~\ref{sec:theory}, we explain how the standard techniques (weight decay, auxiliary self-supervision, and data augmentation) improve the separation between known and unknown class representations, thereby improving the OSR performance. Our final model is combined with these techniques.

\noindent \textbf{Data Augmentation.}
The training data is usually limited. Hence, directly applying metric learning to the raw data without augmentation results in suboptimal inter-class separation and intra-class compactness. The Jacobian norm difference between known and unknown class representations would be negligible in this case. Applying data augmentation resolves this issue by expanding the training set size based on the prior human knowledge of the data. Furthermore, the improved Jacobian norm difference by data augmentation enhances the unknown class detection (Fig.~\ref{fig:comp_aug}).

\noindent \textbf{Weight Decay.} 
Based on \cite{zhang2020deep}, the embedding similarity $s_k$ is optimized based on the gradient \begin{equation}
\partial s_k / \partial \widehat{\boldsymbol{f}}
= 
(\boldsymbol{w}_k - s_k \boldsymbol{f} ) \cdot
\lVert \widehat{\boldsymbol{f}} \rVert_2^{-1}.
\end{equation}
Thus, the small norm $\lVert \widehat{\boldsymbol{f}} \rVert_2$ of the (unnormalized) representation can incite stronger inter-class separation. The weight decay decreases this norm by decreasing the values of the network parameters in $\widehat{\boldsymbol{f}}$ \cite{zhang2018three}. Based on our theory, the enhanced inter-class separation results in higher Jacobian norm values of the unknown class representation, resulting in better separation between the known and unknown in the representation space. The experimental results in Fig.~\ref{fig:comp_wd} precisely verify this theoretical observation.

\noindent \textbf{Auxiliary Self-Supervision.}
To improve the unknown class detection performance, several works \cite{hendrycks2019using,golan2018deep} employ an auxiliary supervision task to predict the degree of rotation (either 0, 90, 180, or 270) on the rotated images. This extra discriminative task poses additional inter-class separation learning on the model. Based on our observations in Sec.~\ref{sec:theory} and \ref{sec:exproof} and Fig.~\ref{fig:fig_k_jnd}, posing additional inter-class separation increases the Jacobian norm of the unknown, thereby improving the separation between the known and unknown class representations (Fig.~\ref{fig:comp_self}). We note, however, that the auxiliary self-supervision should be accompanied with care; predicting rotation in a standard manner may collapse the original class prototypes $w_k$ as the rotation prediction head regards the original classes as a single $0$-degree class. Hence, we add the auxiliary self-supervision loss $\mathcal{L}_{self}$ with a small coefficient $\lambda_{self}=0.1$.

Our final metric-learning objective is to minimize the combined loss $\mathcal{L} + \lambda_{self} \mathcal{L}_{self}$ with data augmentation and weight decay.

\subsection{Inference: Unknown Class Detection by the Sample-Wise Loss Function}
\label{sec:model_score}
To effectively detect unknown class samples during the inference stage, we utilize the sample-wise loss function. Based on our theoretical finding, the loss function is aware of both the Jacobian norm difference and the closeness to the known class prototype:
\begin{alignat}{2}
\label{eq:score_loss}
\mathcal{L}(\boldsymbol{x}) \text{ low/high} 
\Longleftrightarrow & \mathbb{E}_{\boldsymbol{x}_u \sim \mathcal{U}}\left[
\left\lVert \frac{\partial \boldsymbol{f}}{\partial \boldsymbol{x}}(\boldsymbol{x}_u) \right\rVert_2
\right]
- \left\lVert \frac{\partial \boldsymbol{f}}{\partial \boldsymbol{x}}(\boldsymbol{x}) \right\rVert_2 \text{ low/high} && \\
& \text{ and } \min_k \mathcal{D}( \boldsymbol{f}(\boldsymbol{x}), \boldsymbol{w}_k) \text{ low/high} && \nonumber
\end{alignat}
Hence, the loss function (1) differentiates the the known class representations in the low Jacobian norm region from the unknown class representations residing in the region of high Jacobian norm, and (2) separates the known class close to the prototypes $w_k$ from the unknown class instances. The positive correlation indicated in Fig.~\ref{fig:scatter_sample_loss_vs_jn} vindicates the property of loss function described in Eq.~\eqref{eq:score_loss}.

\section{Experiments for Comparison}
\label{sec:expcomp}

\begin{table*}[t!]
\caption{Unknown class detection performance (in AUC) and closed-set accuracy (ACC) for OSR where the unknown class is derived from the same distribution. The results are the averages from 5 random splits.  * indicates that the values are taken from the references. `Arch.' denotes the backbone network used.
}
\label{table:performance_osr_nud}
\centering
\resizebox{.8\linewidth}{!}{
\begin{tabular}{l l cc cc cc cc cc}
\toprule
\multirow{2}{*}{Method}  & \multirow{2}{*}{Arch.} & \multicolumn{2}{c}{CIFAR10} & \multicolumn{2}{c}{CIFAR10+} & \multicolumn{2}{c}{CIFAR50+} & \multicolumn{2}{c}{SVHN} & \multicolumn{2}{c}{TinyImageNet} \\
\cmidrule(r){3-4}\cmidrule(r){5-6}\cmidrule(r){7-8}\cmidrule(r){9-10}\cmidrule(r){11-12}
& & ACC & AUC & ACC & AUC & ACC & AUC & ACC & AUC & ACC & AUC\\
\midrule 
SCE* \cite{hendrycks2016baseline} (ICLR'17) & VGG & - & 67.7 & - & 81.6 & - & 80.5 & - & 88.6 & - & 57.7  \\
OpenMax* \cite{bendale2016towards} (CVPR'16) & VGG & 80.1 & 69.5 & - & 81.7 & - & 79.6 & 94.7 & 89.4 & - & 57.6  \\
RPL* \cite{chen2020learning} (ECCV'20) & VGG & - & 82.7 & - & 84.2 & - & 83.2 & - & 93.4 & - & 68.8  \\
PROSER* \cite{zhou2021learning} (CVPR'21) & VGG & 92.6 & 89.1 & - & 96.0 & - & 95.3 & 96.4 & 94.3 & 52.1 & 69.3  \\
CPN* \cite{yang2020convolutional} (TPAMI'22) & VGG & 92.9 & 82.8 & - & 88.1 & - & 87.9 & 96.4 & 92.7 & - & 63.9  \\
ODL* \cite{liu2022orientational} (TPAMI'22) & VGG & - & 88.5 & - & 91.1 & - & 90.6 & - & 95.4 & - & 74.6 \\
Ours (combined) & VGG & \textbf{96.4} & \textbf{89.5} & \textbf{96.3} & \textbf{96.2} & \textbf{96.5} & \textbf{95.7} & \textbf{97.4} & \textbf{95.7} & \textbf{78.2} & \textbf{75.3} \\
%
\midrule
SCE (ICLR'17) & WRN & 92.1 & 76.5 & 94.0 & 84.7 & 94.0 & 83.8 & 97.0 & 92.1 & 65.8 & 66.1  \\
N-SCE & WRN & 93.7 & 76.8 & 93.7 & 85.3 & 93.7 & 84.4 & 97.1 & 91.8 & 64.5 & 66.4  \\
DOC \cite{shu2017doc} (EMNLP'17) & WRN & 91.6 & 78.0 & 93.9 & 88.2 & 93.9 & 88.1 & 97.0 & 93.5 & 59.6 & 65.5  \\
RPL (ECCV'20) & WRN & 94.7 & 82.0 & 95.9 & 91.1 & 96.1 & 90.9 & 97.5 & 93.4 & 71.4 & 70.4  \\
CPN (TPAMI'22) & WRN & 91.2 & 76.2 & 93.7 & 84.6 & 93.7 & 83.2 & 96.7 & 92.3 & 59.7 & 64.6  \\
Ours (combined) & WRN & \textbf{97.0} & \textbf{89.0} & \textbf{97.7} & \textbf{96.6} & \textbf{97.6} & \textbf{96.0} & \textbf{98.0} & \textbf{97.0} & \textbf{79.1} & \textbf{77.0}  \\
\midrule
CSSR* \cite{huang2022class} (TPAMI'22) & ResNet-18 & - & 91.3 & - & 96.3 & - & 96.2 & - & 97.9 & - & 82.3 \\
GoodOSR \cite{DBLP:journals/corr/abs-2110-06207} (ICLR'22) & ResNet-18 & \textbf{96.3} & 91.4 & \textbf{97.4} & 96.0 & \textbf{97.4} & 94.2 & 96.6 & 97.5 & 83.1 & 82.2   \\
Ours & ResNet-18 & 96.0 & \textbf{92.9} & 97.3 & \textbf{98.0} & 97.3 & \textbf{96.5} & \textbf{96.8} & \textbf{98.2} & \textbf{83.2} & \textbf{82.9} \\
\bottomrule
\end{tabular}
}
\end{table*}

\begin{table*}[t]
\caption{OSR performance under macro-averaged F1-score. * indicates that the values are taken from the references. `Arch.' states the backbone network used. The weight decay is applied in default.
}
\label{table:performance_osr_ood}
\centering
\resizebox{.7\linewidth}{!}{
\begin{tabular}{l l c cccc c}
\toprule
Method & Arch. & Param. & \shortstack{ImageNet-\\crop} & \shortstack{ImageNet-\\resize} & \shortstack{LSUN-\\crop} & \shortstack{LSUN-\\resize} & Avg.\\
\cmidrule{1-8}
SCE* \cite{hendrycks2016baseline} (ICLR'17) & VGG  & 1.1M & 63.9 & 65.3 & 64.2 & 64.7 & 64.5 \\
OpenMax* \cite{bendale2016towards} (CVPR'16) & VGG  & 1.1M & 66.0 & 68.4 & 65.7 & 66.8 & 66.7 \\
CROSR* \cite{yoshihashi2019classification} (CVPR'19) & VGG  & 1.1M & 72.1 & 73.5 & 72.0 & 74.9 & 73.1 \\
GFROSR* \cite{perera2020generative} (CVPR'20) & VGG  & 1.1M & 75.7 & 79.2 & 75.1 & 80.5 & 77.6 \\
PROSER* \cite{zhou2021learning} (CVPR'21) & VGG  & 1.1M & \textbf{84.9} & 82.4 & \textbf{86.7} & 85.6 & 84.9 \\
OvRN-CD* \cite{jang2022collective} (TNNLS'22)  & VGG & 1.1M & 83.5 & 82.5 & 84.6 & 83.9 & 83.6 \\
%
\rowcolor{finesky}
Ours & VGG  & 1.1M & 84.2 & \textbf{88.4} & 85.1 & \textbf{88.1} & \textbf{86.5} \\
\cmidrule{1-8}
SCE          & WRN-16-4 & 2.7M & 79.1 & 79.2 & 80.3 & 80.8 & 79.9 \\
\rowcolor{lightgray}
m-OvR          & WRN-16-4 & 2.7M & 80.5 & 79.8 & 79.2 & 81.2 & 80.2 \\
SCE  + A     & WRN-16-4 & 2.7M & 84.5 & 88.5 & 87.0 & 88.8 & 87.2 \\
\rowcolor{lightgray}
m-OvR + A      & WRN-16-4 & 2.7M & 87.4 & 89.0 & 89.0 & 90.0 & 88.9 \\
SCE + A + S & WRN-16-4 & 2.7M & 84.6 & 87.5 & 87.5 & 87.5 & 86.8 \\
\rowcolor{finesky}
m-OvR + A + S          & WRN-16-4 & 2.7M & \textbf{89.1} & \textbf{90.5} & \textbf{90.4} & \textbf{90.9} & \textbf{90.2} \\
\bottomrule
\end{tabular}
}
\end{table*}

\begin{table*}[t]
\caption{Ablation of our proposed model by analyzing its training components: the m-OvR loss, unit-normalization (N) of representations, the margin $m$ in the similarity computation, weight decay (W), data augmentation (A), and auxiliary self-supervision (S).
Models are evaluated in terms of unknown class detection performance (AUC), closed-set accuracy (ACC), and detection accuracy (DetACC). SCE substitutes in the absence of m-OvR.
}
\label{table:ablation}
\centering
\resizebox{.99\linewidth}{!}{
\begin{tabular}{c c | c|c|c|c|c|c |c  cccccc}
\toprule
& Model    & OvR         & N          & $m$             & W          & A          & S          & \shortstack{CIFAR10 \\ AUC/ACC/DetACC}      & \shortstack{CIFAR10+ \\ AUC/ACC/DetACC}      & \shortstack{CIFAR50+ \\ AUC/ACC/DetACC}     & \shortstack{SVHN \\ AUC/ACC/DetACC}         & \shortstack{TIN \\ AUC/ACC/DetACC}           & \shortstack{Avg. \\ AUC/ACC/DetACC}         \\
\cline{1-14}
& Baseline &            &            &                        &            &            &            & 76.5/92.1/70.8 & 84.7/94.1/78.9 & 83.9/94.1/78.1 & 92.1/97.1/86.1 & 66.1/65.8/62.8 & 80.7/88.6/75.3 \\
\cline{2-8}
& m-OvR & \checkmark & \checkmark & \checkmark &            &            &            & 79.3/91.6/72.8 & 90.1/93.9/82.6 & 89.8/94.1/82.1 & 93.8/97.1/87.6 & 66.2/62.3/62.7 & 83.9/87.8/77.6 \\
\cline{1-14}
\multirow{4}{*}{\shortstack{one\\out}}
& a        &            &            &                        & \checkmark & \checkmark & \checkmark & 85.0/96.1/78.1 & 90.4/96.8/84.1 & 89.3/96.9/82.6 & 94.4/97.7/88.7 & 74.7/77.8/70.0 & 86.7/93.1/80.7 \\
\cline{2-8}
& b        &            & \checkmark &                        & \checkmark & \checkmark & \checkmark & 83.9/95.9/77.0 & 92.9/97.0/85.8 & 91.7/97.1/85.0 & 95.6/97.7/88.6 & 75.0/77.0/69.8 & 87.8/92.9/81.2 \\
\cline{2-8}
& c        & \checkmark & \checkmark & \checkmark &             & \checkmark & \checkmark & 79.2/93.1/63.7 & 95.7/95.6/89.0 & 94.5/95.5/86.8 & 95.5/97.1/88.3 & 66.1/67.0/62.8 & 86.2/89.7/78.1 \\
\cline{2-8}
& d        & \checkmark & \checkmark & \checkmark & \checkmark &            & \checkmark & 81.7/95.1/77.5 & 92.2/95.7/86.5 & 90.6/95.7/85.1 & 87.1/97.6/88.4 & 65.6/66.4/63.8 & 83.4/90.1/80.3 \\
\cline{1-14}
\multirow{3}{*}{\shortstack{w/o\\A+S}}
& e        &            &            &                       & \checkmark &            &            & 81.6/94.1/75.7 & 87.7/95.6/81.3 & 87.1/95.8/80.7 & 94.4/97.7/89.3 & 70.1/67.8/65.4 & 84.3/90.2/78.5 \\
\cline{2-8}
& f        &            & \checkmark &                       & \checkmark &            &            & 79.9/93.4/73.8 & 86.8/95.0/80.7 & 85.0/94.9/78.7 & 92.7/97.3/87.0 & 67.3/67.0/63.6 & 82.3/89.6/76.7 \\
\cline{2-8}
& g        & \checkmark & \checkmark & \checkmark  & \checkmark &            &            & 80.3/93.8/74.1 & 89.6/95.2/82.8 & 88.6/95.0/81.4 & 92.8/97.4/86.3 & 68.9/65.4/64.7 & 84.0/89.4/77.9 \\
\cline{1-14}
\multirow{3}{*}{\shortstack{w/o\\S}}
& h        &            &            &                       & \checkmark & \checkmark &            & 83.2/95.9/76.5 & 91.1/97.1/84.6 & 90.6/97.0/84.1 & 94.1/97.5/87.7 & 75.6/76.5/69.9 & 86.9/92.8/80.5 \\
\cline{2-8}
& i        &            & \checkmark &                       & \checkmark & \checkmark &            & 84.6/96.1/77.0 & 95.6/97.2/88.8 & 94.6/97.2/86.2 & 96.6/97.9/90.5 & 76.8/77.8/70.8 & 89.6/93.2/82.7 \\
\cline{2-8}
& j        & \checkmark & \checkmark & \checkmark  & \checkmark & \checkmark &            & 84.6/96.1/77.0 & 95.6/97.2/88.8 & 94.6/97.2/86.2 & 96.6/97.9/90.5 & 76.8/77.8/70.8 & 89.6/93.2/82.7 \\
\cline{1-14}
\multirow{3}{*}{\shortstack{on\\$m$}}
& k        & \checkmark & \checkmark &                       & \checkmark & \checkmark & \checkmark & 85.0/96.4/78.6 & 92.9/97.2/86.3 & 92.4/97.2/85.4 & 95.8/98.0/89.9 & 76.8/79.2/71.3 & 88.6/93.6/82.3 \\
\cline{2-8}
& Ours    & \checkmark & \checkmark & \checkmark  & \checkmark & \checkmark & \checkmark & 87.7/97.2/80.0 & 96.6/97.8/89.6 & 96.0/97.6/88.1 & 97.0/98.0/91.5 & 77.0/79.1/71.0 & \textbf{90.9}/\textbf{93.9}/\textbf{84.0}  \\
\bottomrule
\end{tabular}
}
\end{table*}

\begin{figure*}[t]
\centering
\includegraphics[width=0.85\linewidth]{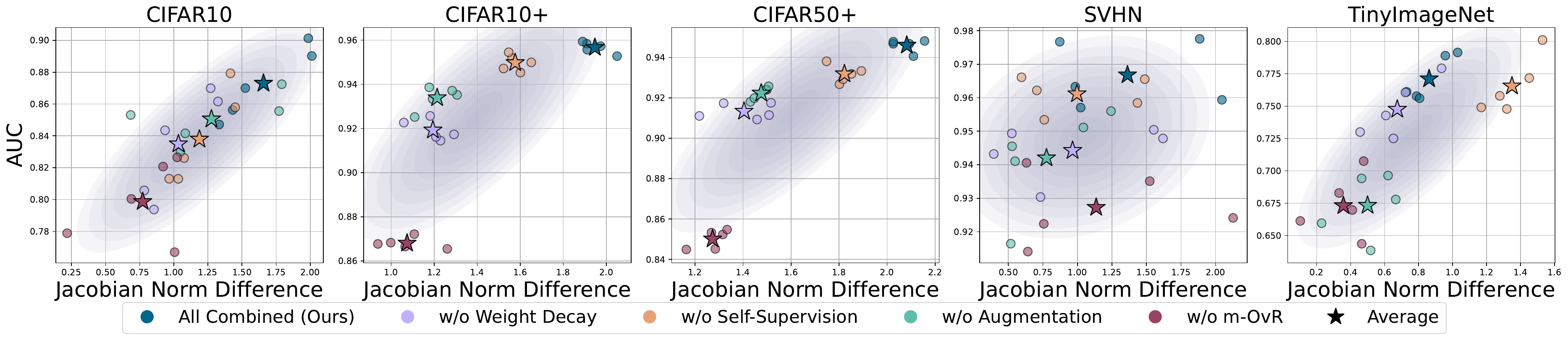}
\caption{The plot of Jacobian norm difference versus the unknown class detection performance (AUC). Each point corresponds to a distinct model trained over a different known set, following the protocol of \cite{neal2018open}.  Different colors indicate different methods. The plot shows the positive correlation, and the all combined model has the largest Jacobian norm differnce.
}
\label{fig:gnd_vs_auc}
\end{figure*}

The experiment section is outlined as follows: 
(1) We compare our method with other baseline OSR models for the unknown class detection task under two different widely-used protocols, Protocol A \cite{neal2018open} and Protocol B \cite{yoshihashi2019classification}.
(2) We conduct a careful ablation study of our method, analyzing each component in terms of the unknown class detection performance and the Jacobian norm. (3) We visualize and analyze the Jacobian norm of representation with respect to the metric distances in the representation space. To this end, we compare our proposed model with a baseline model trained with the bare SCE loss.

Our proposed model is trained with the m-OvR loss in all experiments below. Unless specified, we always include weight decay, data augmentation, and auxiliary self-supervision in our model. The default model hyperparameters are as follows: the scale term $T=32$, margin $m=0.5$, the auxiliary self-supervision coefficient $0.1$, and the weight decay $1e{-}3$. 

We consider three backbones to extract the representation:  WRN-16-4 \cite{zagoruyko2016wide}, VGG \cite{neal2018open}, and ResNet-18. For WRN-16-4 and VGG, our model is trained by SGD with $20$k training iterations unless specified otherwise. Its learning rate is regulated under a cosine scheduler, initiating from $0.1$ and decaying to $1e{-}5$. The batch size is 128. 
In the case of the ResNet-18 backbone, on the other hand, the model is trained for 200 epochs under the SGD optimizer with a momentum of 0.9 and a learning rate of 0.06 that decays to 0 by the cosine learning scheduler. 

In all experiments, the model is trained only with known classes so that the model never sees any unknown class sample during training.

\subsection{Performance Comparison - Protocol A}
\label{sec:exp_data}

\noindent
\textbf{Datasets-Protocol A.}
In this protocol \cite{neal2018open}, we use five different OSR datasets to compare different OSR methods in terms of the closed-set classification accuracy and unknown class detection performance.

Our method is evaluated for unknown class detection performance (AUC) and closed-set accuracy (ACC).
The protocol used in \cite{neal2018open} is adopted with the following benchmark datasets:

\begin{itemize}
\item CIFAR10 and SVHN: Among the total ten classes, $K{=}6$ classes are chosen as the known ones, regarding the rest as a single unknown class. CIFAR10 \cite{krizhevsky2009learning} consists of generic object images while SVHN \cite{netzer2011reading} of street view numbers.

\item CIFAR10+ and CIFAR50+: To make CIFAR10 more challenging, CIFAR10+ and CIFAR50+ are considered, in which $K{=}4$ known classes are selected from CIFAR10 while 10 (or 50) classes from CIFAR100 \cite{krizhevsky2009learning} constitute a single unknown class.

\item TinyImageNet: In TinyImagenet (TIN) \cite{le2015tiny} with more diverse categories, $K {=} 20$ classes constitutes the known, while the other 180 remaining ones form a single unknown class.
\end{itemize}

\noindent
\textbf{Results-Protocol A.}
The comparison results are given in Table \ref{table:performance_osr_nud}, which indicate that our proposed methodology is effective for OSR across different backbone architectures, including VGG, WRN, and ResNet-18.

A significant attribute of our methodology lies in the employment of our margin-based loss, m-OvR, which not only  optimizes intra-class compactness but also ensures inter-class separation by circumventing inter-class collapse, as detailed in Prop. \ref{prop:ovr_margin}. This aspect renders our work as an improvement over the prevailing techniques such as RPL, CPN, and OvRN-CD, which predominantly focus on the inter-class aspects alone.
Furthermore, our methodology incorporates carefully chosen subsidiary techniques, including weight decay, representation unit-normalization, and self-supervision through rotation prediction, which can efficaciously enhance OSR.

We note that our approach, even without the use of complex training tricks but solely utilizing the m-OvR loss, is comparable to the state-of-the-art GoodOSR. The pivotal differentiation lies in that GoodOSR boosts the OSR performance by excessive hyperparameter tuning and various cutting-edge training tricks, while ours is simply based on the loss function design.

\subsection{Performance Comparison - Protocol B}

\noindent
\textbf{Datasets-Protocol B.}
In this experiment, the model is trained over $K$ known classes and classifies $K{+}1$ where the $K+1$-th class is the unknown class. The protocol given in \cite{yoshihashi2019classification} is adopted.
For benchmarking, we use CIFAR10 classes as the known with $K{=}10$. The unknown class is either ImageNet \cite{russakovsky2015imagenet} or LSUN \cite{yu2015lsun} that comprises scenery images. They are resized or cropped, constituting ImageNet-crop, ImageNet-resize, LSUN-crop, or LSUN-resize.
Following the convention given in \cite{zhou2021learning,yoshihashi2019classification}, we choose the threshold $\tau$ for the inference score in Sec.~\ref{sec:model_score} so that $10\%$ of the validation set is detected as unknown class samples. The performance is evaluated using macro-averaged F1-score \cite{macro_f1}.

\noindent
\textbf{Results-Protocol B.}
The result in Table \ref{table:performance_osr_ood} shows that our proposed method outperforms all other baselines in the average performance. Under the WRN-16-4 architecture, m-OvR shows superiority over SCE, significantly more effective than SCE when applied with augmentation (A) and self-supervision (S). This is mainly due to 
the large Jacobian norm difference derived from the highly discriminative representations of the m-OvR (as observed in Fig.~\ref{fig:comp_loss}) triggers a strong separation between the known and unknown class representations.

\begin{figure}[t]
\centering
\includegraphics[width=.85\linewidth]{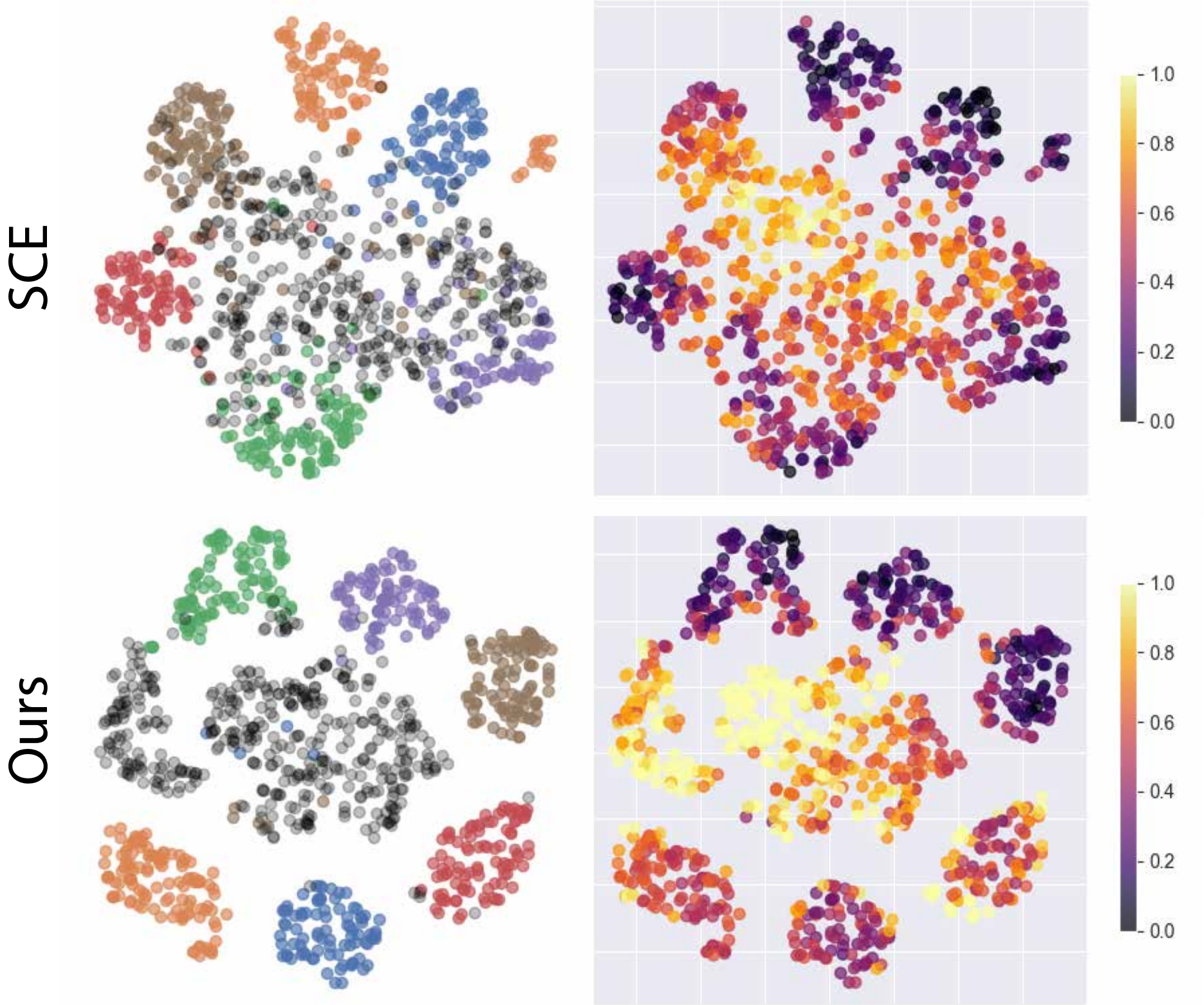}
\caption{The $2$-dimensional t-SNE \cite{van2009learning} visualization of $\boldsymbol{f}(\boldsymbol{x})$ trained on MNIST under the protocol of \cite{neal2018open}. In the left column, the black color denotes the unknown class. The temperature in the heat map (right column) indicates the (min-max normalized) Jacobian norm $\lVert \partial \boldsymbol{f} / \partial \boldsymbol{x} \rVert_F$. The figure shows that the larger the Jacobian norm difference between the known and unknown (i.e., the color contrast in the right column figures), the better the separation between the known and unknown.
}
\label{fig:heat}
\end{figure}

\subsection{Ablation Study}
\label{sec:exp_abl}

\noindent
\textbf{Ablation on Training Components.}
Each component in our model is more carefully evaluated in this experiment. For this purpose, we use the standard metrics used in OSR; namely, AUC for the unknown class detection performance, the closed-set accuracy (ACC), and detection accuracy (DetACC) \cite{lee2018simple}.
The second block in the row shows that the m-OvR loss outperforms the SCE loss by a large margin, even when there is no data augmentation (A), weight decay (W), and self-supervision (S).
The representation embedding normalization (N) improves the performance by preventing trivial increase of the Jacobian norm.
The third block in a row (`one out') of Table \ref{table:ablation} along with the model-j compares each component by removing one of them out, verifying the effectiveness of each in the entire model.  When the standard data augmentation is available (i.e.~the fifth block), m-OvR effectively utilizes the data, thus more effectively separating the known from the unknown than SCE. Finally, the sixth block analyzes the margin, which improves the effectiveness of the loss-based unknown class detector by resolving the prototype misalignment issue.

\noindent
\textbf{Ablation with Jacobian Norm Difference}
The scatter plot for each fixed dataset in Fig.~\ref{fig:gnd_vs_auc} shows that the degree of separation between the known and unknown class representations positively correlates to the Jacobian norm difference. The correlations in CIFAR10 and TinyImageNet are strong, while CIFAR10+ and CIFAR50+ exhibit some degree of nonlinearity. In SVHN, on the other hand, the correlation is comparatively weak due to the performance saturation. Moreover, this proves that the large Jacobian norm difference is not the only factor that captures distance separation between the known and unknown, as already remarked by Sec.~\ref{sec:theory_limit}).

\begin{figure}[t]
\centering
\subfloat[]{
\includegraphics[width=.45\linewidth]{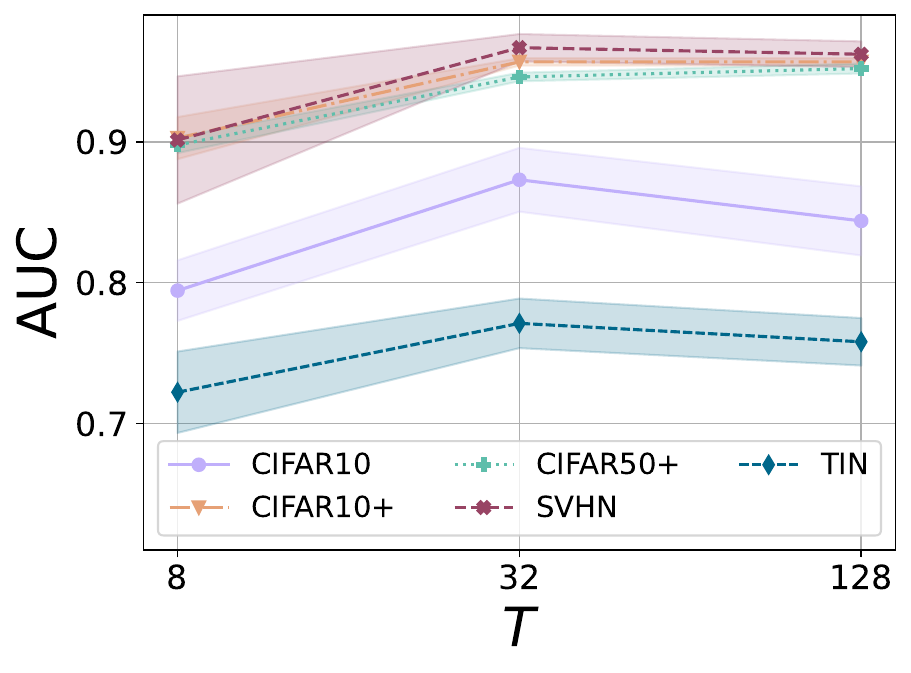}
}
\subfloat[]{
\includegraphics[width=.45\linewidth]{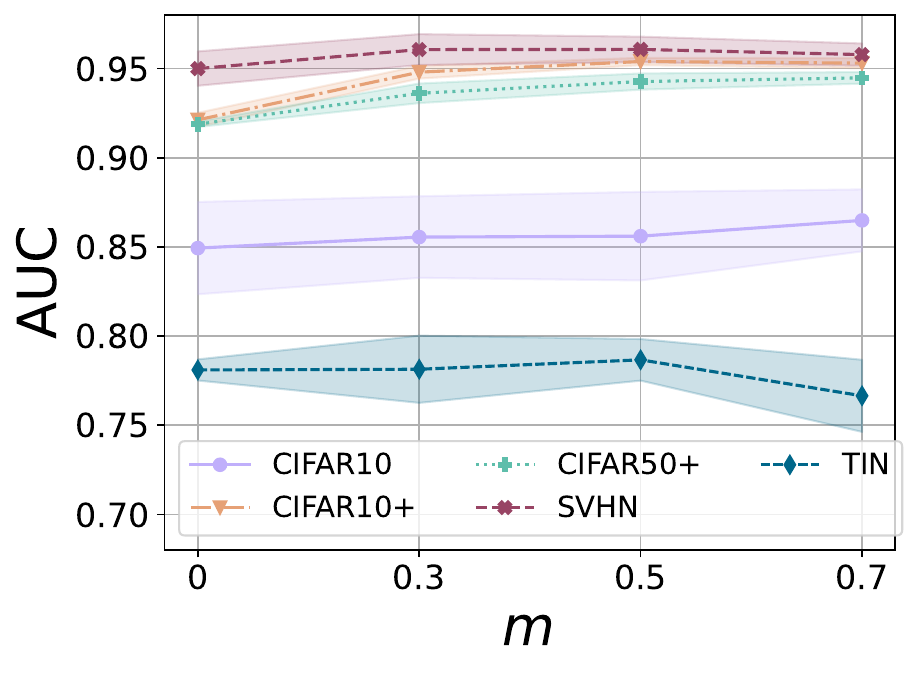}
}
\\
\subfloat[]{
\includegraphics[width=.45\linewidth]{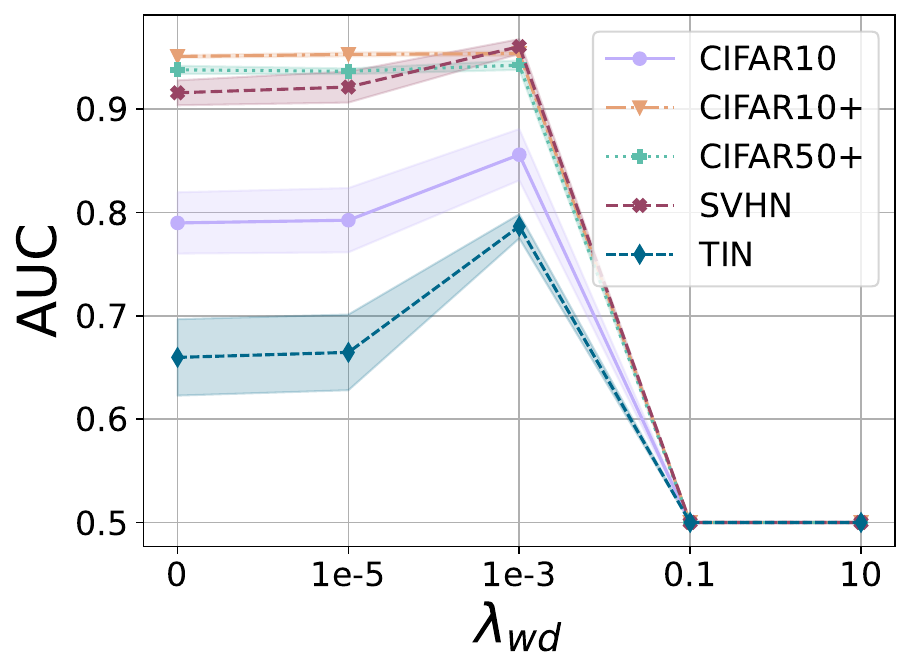}
}
\subfloat[]{
\includegraphics[width=.45\linewidth]{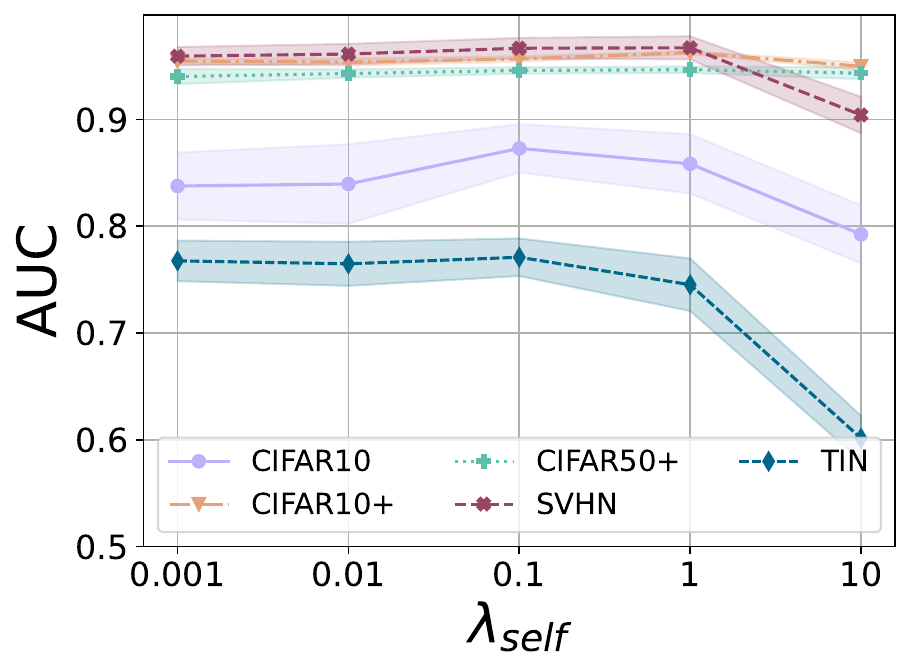}
}
\caption{
Unknown class detection performance (AUC) versus (a) the scale $T$, (b) the margin $m$, (c) the coefficient of the weight decay, and (d) the coefficient of the auxiliary self-supervision loss.
}
\label{fig:sub_hyp}
\end{figure}

\noindent
\textbf{Ablation on Model Hyperparameters.}
We analyze the hyperparameters of our overall model. Fig.~\ref{fig:sub_hyp} shows that the unknown class detection performance is robust for a sufficiently large scale term $T$, and the margin $m$ should not be too large. 

On the other hand, if the weight decay coefficient $\lambda_{wd}$ is overly large, then it collapses the embedding to a constant (i.e., zero vector). At the same time, overly small $\lambda_{wd}$ has no impact as a regularizer. Finally, we remark that selecting a proper coefficient for the weight decay is not tricky by observing the train loss dynamic during the early stage.

As already remarked in Sec.~\ref{sec:model_sub}, the rotation-based self-supervision auxiliary loss contributes positively only when its coefficient $\lambda_{self}$ is small (i.e., smaller than $1$). The unknown class detection performance is robust for the small values of $\lambda_{self}$.

\subsection{Visual Analysis of the Jacobian Norm of Representation}

In the 2-dimensional visualization of Fig.~\ref{fig:heat} obtained by applying t-SNE on the embedding representations of data samples, the known classes exhibit small Jacobian norm values while the unknown samples have larger Jacobian norm values. Moreover, the degree of distance-wise separation becomes high when the Jacobian norm contrast between the known and unknown classes is more vivid.

\section{Conclusion}
We have demonstrated that closed-set metric learning distinguishes the unknown from the known by causing their representations' Jacobian norm values to differ. Crucially, inter-class learning serves as the primary factor in this process, as it modifies the unknown class samples' representations without directly accessing them. Recognizing the significant role of inter-class learning in OSR, we developed a marginal one-vs-rest loss function designed to promote robust inter-class separation. By integrating this loss with other techniques that amplify the Jacobian norm disparity between known and unknown classes, we have successfully showcased the efficacy of our method on standard OSR benchmarks.

\paragraph{Acknowledgment}
\sloppy
This work was supported by the National Research Foundation of Korea (NRF) grant funded by the Korea government (MSIP) (NO. NRF-2022R1A2C1010710).

%% file: supplementary.tex
\clearpage

\appendix

\section{Assumptions to the Theory}
\label{sec:assumption_theory}
For technical proofs of the theoretical parts, we assume the following:
Firstly, the known classes $C_i$ in the input space (e.g., image pixel space) follow the following regularity:

\begin{ass}
\label{ass:regular_class}
$\quad$
\begin{enumerate}
\item[(a)]
Each known class $C_k $ is a simple smooth, connected compact manifold with a nonzero volume $\Vol(C_k)$.
\item[(b)]
For any linear path $\boldsymbol{\gamma}: [0,1] \to \mathcal{X}$ from $C_j$ to $C_k$, there exist $t_1$ and $t_2$ such that $\boldsymbol{\gamma}([t_1, t_2]) \subseteq \mathcal{O}$ with $0< t_1 < t_2 < 1$
\end{enumerate}
\end{ass}

Note that the volume, denoted by $\Vol$, in Assumption \ref{ass:regular_class}a is the Lebesgue measure in the Euclidean space $\mathbb{R}^d$.
Assumption \ref{ass:regular_class}b indicates that the in-between part of the linear interpolation between different known classes is a part of an open set. This assumption is reasonable since interpolating in a very high-dimensional space, such as image pixel space, always induces meaningless inputs in the middle part of the interpolation.

Now, as to the representation embedding function $\boldsymbol{f}$, we restrict our consideration to the neural network family with the following regularity. 

\begin{ass}
\label{ass:regular_embedding}
$\quad$
\begin{enumerate}
\item[(a)]
$\boldsymbol{f}$ is a bounded smooth parametrized neural network (i.e., $\boldsymbol{f} = \boldsymbol{f}_\theta$ with a parameter $\theta$) with a sufficient complexity.

\item[(b)]
For any linear simple smooth path $\boldsymbol{\gamma}: [0,1] \to \mathbb{R}^{d_z}$ from $\boldsymbol{x}_j \in C_j$ to $\boldsymbol{x}_k \in C_k$, the inter-class distance maximization \textbf{strictly increases} the length of the path $\boldsymbol{f}(\boldsymbol{\gamma}([0,1]) \cap \mathcal{O})$ between $\boldsymbol{f}(\boldsymbol{x}_j)$ and $\boldsymbol{f}(\boldsymbol{x}_k)$.

\item[(c)]
For any simple smooth path $\boldsymbol{\gamma}:[0,1] \to \mathbb{R}^{d_z}$ from $\boldsymbol{x}_j \in C_j$ to $\boldsymbol{x}_k \in C_k$, 
the inter-class distance maximization does \textbf{not decrease} the length of any sub-path 
$\boldsymbol{f}(\boldsymbol{\gamma}([t_1,t_2]) \cap \mathcal{O})$ between $\boldsymbol{f}(\boldsymbol{x}_j)$ and $\boldsymbol{f}(\boldsymbol{x}_k)$ for any $0 \leq t_1 < t_2 \leq 1$


\end{enumerate}
\end{ass}

Assumption \ref{ass:regular_embedding}a is a standard regularity condition.
Assumption \ref{ass:regular_embedding}b means that the inter-class separation is effective on the linear interpolating path $\gamma$. Assumption \ref{ass:regular_embedding}b is visualized in Fig.~\ref{fig:concept}, which indicates that the length of projected (inter-class) path is strictly increased by metric learning.
Assumption \ref{ass:regular_embedding}c means that the inter-class distance maximization involves no contradictory behavior when the inter-class path is observed locally. The assumption is reasonable and based on the empirical evidence given in Fig.~\ref{fig:interpolate}.

In the below mathematical derivations, when we say that a quantity $Q(\boldsymbol{f})$ increases with respect to a function $\boldsymbol{f} : \mathbb{R}^{d} \to \mathbb{R}^{d_z}$, it formally means that there is a sequence $(\boldsymbol{f}^{(n)})_{n=0}^N$ of functions $\boldsymbol{f}^{(n)}: \mathbb{R}^d \to \mathbb{R}^{d_z}$ with $N \geq 1$ such that
\begin{equation}
Q(\boldsymbol{f}^{(n)}) \leq Q(\boldsymbol{f}^{(n+1)})
\end{equation}
for all $0 \leq n <N$.
In the case of a strict increment, the inequality is replaced by the strict one.
The decrement of $Q(\boldsymbol{f})$ is similarly defined.

Depending on the context, $Q$ may include the vectors $\boldsymbol{w}_k \in \mathbb{R}^{d_z}$ (that serve as representation prototypes in our work): $Q = Q(\boldsymbol{f}, \{\boldsymbol{w}_k\}_{k=1}^K)$.

\section{Proofs to the Theory}
\label{sec:proof_theory}

\begin{proof}[Proof of Proposition \ref{prop:collapse}]
Fix $C_k$. We prove a stronger result that 
\begin{equation}
\frac{\partial \boldsymbol{f}}{\partial \boldsymbol{x}} \to \mathbf{0} \in \mathbb{R}^{d_z \times d}
\end{equation}
for all $\boldsymbol{x} {\in} C_k$. In which case, 
$
\frac{d \boldsymbol{f}(\boldsymbol{\gamma}(t))}{dt} \to \textbf{0}
$
for all $t \in \boldsymbol{\gamma}^{-1}( \boldsymbol{\gamma}([0,1]) \cap C_k)$ since
$
\frac{d f_j ( \boldsymbol{\gamma} (t) )}{dt} = \sum_{i=1}^{n} \frac{\partial f_j}{\partial x_i}(\boldsymbol{\gamma}(t))  \gamma'_i(t)
$
for all $j$ where $\boldsymbol{x} = (x_1, \dots, x_d)$ and $\boldsymbol{f} = (f_1, \dots, f_{d_z})$.
Then, this implicates that the length of $\boldsymbol{f}(\boldsymbol{\gamma}([0,1]) \cap C_k)$ converges to $0$ (as the pointwise convergence guarantees the $L_p$ convergence when the functions are bounded).

Let $\boldsymbol{f}^{(n)} {:=} \boldsymbol{f}_{\theta^{(n)}}$ be a sequence that minimizes $\D(\boldsymbol{f}^{(n)}(\boldsymbol{x}), \boldsymbol{w}_k)$ to $0$ as $n {\to} N$. 
Let $ \boldsymbol{x} {\in} C_k$. Since the quantity at hand is a partial derivative, without loss of generality, assume $\boldsymbol{f}(\boldsymbol{x})=f(x)$ and $\boldsymbol{x}=x$ are scalar-valued (and also for $\boldsymbol{w}_k = w_k$). Fix $\epsilon > 0$. Then for some $\delta > 0$, we have
\begin{equation}
\left\lvert
\dfrac{d }{dx} f^{(n)}(x) - \left(
\dfrac{f^{(n)}(x+h) - f^{(n)}(x)}{h}
\right)
\right\rvert 
< \epsilon
\end{equation}
for all $h \in [-\delta, \delta] \setminus \{ 0 \}$. Taking $n \to N$, we obtain
\begin{equation}
\left\lvert
\lim_{n \to N} \dfrac{d }{dx} f^{(n)}(x) - 0
\right\rvert 
< \epsilon
\end{equation}
since $f^{(n)} \to w_k$.
The arbitrariness of $\epsilon$ concludes the proof.
\end{proof}

\begin{proof}[Proof of Theorem \ref{thm:grad_norm}]
The intra-class minimization part is proved in the proof of Proposition \ref{prop:collapse}.

For the inter-class maximization part, without loss of generality, redefine $\boldsymbol{f}$ such that $\boldsymbol{f}(\boldsymbol{x}) = 0$ for all $\boldsymbol{x} \in \mathcal{X} \setminus \mathcal{O}$, while to be the same as the original $\boldsymbol{f}$ over $\mathcal{O}$. Now, it suffices to prove the strict increment of $\int_{\mathcal{X}} \lVert \frac{\partial \boldsymbol{f} (\boldsymbol{x})}{\partial \boldsymbol{x}} \rVert_F \; d\boldsymbol{x}$ with respect to this $\boldsymbol{f}$. 

Note that Assumption \ref{ass:regular_embedding}c implies  that $\lVert \frac{\partial \boldsymbol{f} \circ \boldsymbol{\gamma}}{\partial t} \rVert_2$ as a function of $t$ is non-decreasing with respect to the changing $f$ for any simple smooth path $\gamma$. Hence, 
$\lVert \frac{\partial \boldsymbol{f}}{\partial x_l}(\boldsymbol{x}) \rVert_2^2$
is non-decreasing for any $l=1,\dots,d$. We use this property freely in the following. 

Since $\Vol(C_i) >0$ for all known classes $C_i$, for any pair of different known classes $C_i$ and $C_j$, we have a $(d{-}1)$-dimensional hyperplane $P \subseteq \mathcal{X}$ such that 
\begin{equation}
\Vol_{d-1} \left( \rho(C_i) \cap \rho(C_j) 
\right)>0
\end{equation}
where $\Vol_{d-1}$ is the $(d-1)$-dimensional volume, and $\rho(C_i)$ is the projection of $C_i$ to the hyperplane $P$. Since a coordinate change under rotation and translation does not change the volume integral of a function, we assume without loss of generality that 
\begin{equation}
\boldsymbol{e}_k \perp P;
\end{equation}
that is, $\boldsymbol{e}_k$ is perpendicular to $P$ where $\boldsymbol{e}_k$ is the $k$-th standard basis element of $\mathbb{R}^d$. 

Now, observe 
\begin{equation}
\int_{\mathcal{X}}
\left\lVert
\frac{\partial \boldsymbol{f} (\boldsymbol{x})}{\partial \boldsymbol{x}} 
\right\rVert_F^2
d\boldsymbol{x} 
=
\sum_{l=1}^d
\int_{\mathcal{X}}
\left\lVert
\frac{\partial \boldsymbol{f} (\boldsymbol{x})}{\partial x_l} 
\right\rVert_2^2
d\boldsymbol{x}.
\end{equation}
Since $
\int_{\mathcal{X}}
\lVert
\frac{\partial \boldsymbol{f} (\boldsymbol{x})}{\partial x_l} 
\rVert_2^2
\, d\boldsymbol{x}
$
is non-decreasing for $l \neq k$, it suffices to show that 
$
\int_{\mathcal{X}}
\lVert
\frac{\partial \boldsymbol{f} (\boldsymbol{x})}{\partial x_k} 
\rVert_2^2
\, d\boldsymbol{x}
$
is strictly increasing. Let 
\begin{equation}
R = \{ \widehat{\boldsymbol{x}}_k \in [-1,1]^{d-1} : \boldsymbol{x} \in \rho(C_i) \cap \rho(C_j)\}.
\end{equation}
where $\widehat{\boldsymbol{x}}_k$ denotes $\widehat{\boldsymbol{x}}_k {:=} (x_1, \dots, x_{k-1}, x_{k+1}, \dots, x_d) $ that removes the $k$-th element of $\boldsymbol{x}$.
Note that
\begin{multline}
\int_{\mathcal{X}}
\left\lVert
\frac{\partial \boldsymbol{f} (\boldsymbol{x})}{\partial x_k} 
\right\rVert_2^2
d\boldsymbol{x}
=
\int_{R}
\int_{-1}^1
\left\lVert
\frac{\partial \boldsymbol{f} (\boldsymbol{x})}{\partial x_k} 
\right\rVert_2^2
dx_k d\widehat{\boldsymbol{x}}_k
\\
+
\int_{R^c}
\int_{-1}^1
\left\lVert
\frac{\partial \boldsymbol{f} (\boldsymbol{x})}{\partial x_k} 
\right\rVert_2^2
dx_k d\widehat{\boldsymbol{x}}_k
\label{eq:proof01}
\end{multline}
where $d \widehat{\boldsymbol{x}}_k {:=} dx_1 \cdots dx_{k-1} dx_{k+1} \cdots dx_d$.
Since the second term on the RHS of the above equation is non-decreasing, we consider the first only, whose inner term
\begin{equation}
\int_{-1}^1
\lVert
\frac{\partial \boldsymbol{f} (\boldsymbol{x})}{\partial x_k} 
\rVert_2^2
\, dx_k.
\end{equation}
is decomposed into
\begin{equation}
\int_{-1}^a + \int_{a}^b + \int_b^1
\left\lVert
\frac{\partial \boldsymbol{f} (\boldsymbol{x})}{\partial x_k} 
\right\rVert_2^2
dx_k
\label{eq:proof02}
\end{equation}
where the scalars $a=a(\widehat{\boldsymbol{x}}_k)$ and $b=b(\widehat{\boldsymbol{x}}_k)$ are the infimum and supremum of $\{x_k : \boldsymbol{x} \in C_i \cup C_j \}$, respectively, with $\widehat{\boldsymbol{x}}_k \in R$. The first and third terms non-decrease, thus ignored. 
To compute the mid term
$
\int_{a}^b 
\lVert
\frac{\partial \boldsymbol{f} (\boldsymbol{x})}{\partial x_k} 
\rVert_2^2
dx_k
$, 
consider a path
\begin{equation}
\boldsymbol{\gamma}(t)
=
(x_1,\dots, x_{k-1}, a + (b-a)t, x_{k+1}, \dots, x_d)
\end{equation}
that depends on $\widehat{\boldsymbol{x}}_k = (x_1, \dots, x_{k-1}, x_{k+1}, \dots, x_d) {\in} R$. 
Then, $\boldsymbol{\gamma}$ is a path from $C_i$ to $C_j$ or the other way, and
\begin{alignat}{2}
\int_{a}^b 
\left\lVert
\frac{\partial \boldsymbol{f} (\boldsymbol{x})}{\partial x_k} 
\right\rVert_2^2
dx_k
& 
= (b-a)
\int_0^1
\left\lVert
\frac{d \boldsymbol{f} \circ \boldsymbol{\gamma} (t)}{dt} 
\right\rVert_2^2
dt
\\
& 
= 
(b-a)
\ell( \boldsymbol{f} \circ \boldsymbol{\gamma} )
\end{alignat}
where $\ell( \boldsymbol{f} \circ \boldsymbol{\gamma} ) = \int_0^1
\lVert
\frac{d \boldsymbol{f} \circ \boldsymbol{\gamma} (t)}{dt} 
\rVert_2^2
dt$.
In summary,
\begin{equation}
\label{eq:essence}
\int_{\mathcal{O}} \left\lVert \frac{\partial \boldsymbol{f}}{\partial \boldsymbol{x}}(\boldsymbol{x}) \right\rVert_2^2 d\boldsymbol{x}
= A(\boldsymbol{f}) 
+ 
\int_{R} [b(\widehat{\boldsymbol{x}}_k)-a(\widehat{\boldsymbol{x}}_k)] \ell(\boldsymbol{f} \circ \boldsymbol{\gamma}) d \widehat{\boldsymbol{x}}_k.
\end{equation}
Here, the inter-class maximization does not change $R$, $\widehat{\boldsymbol{x}}_k$, $\gamma$, $a(\widehat{\boldsymbol{x}}_k)$, and $b(\widehat{\boldsymbol{x}}_k)$. On the other hand, the inter-class maximization does not decrease the term $A(\boldsymbol{f})$. Moreover, note that $\gamma \cap \mathcal{O}$ is not empty and contains an interval by Assumption \ref{ass:regular_class}b. Thus, by Assumption \ref{ass:regular_embedding}b, the inter-class maximization strictly increases the term $\ell( \boldsymbol{f} \circ \boldsymbol{\gamma} )$, thereby strictly increasing the global integral of the Jacobian norm over the open set. This finishes the proof.
\end{proof}

\begin{proof}[Proof of Corollary \ref{cor:support}]
By the above theorem, the inter-class distance maximization increases the integral $\int_{\mathcal{O}} \lVert \frac{\partial \boldsymbol{f} (\boldsymbol{x})}{\partial \boldsymbol{x}} \rVert_F \; d\boldsymbol{x}$.
Now, the integral $\int_{\mathcal{O}} \lVert \frac{\partial \boldsymbol{f} (\boldsymbol{x})}{\partial \boldsymbol{x}} \rVert_F \; d\boldsymbol{x}$ can be decomposed into
\begin{equation}
\int_{\mathcal{O}}
\lVert
\tfrac{\partial \boldsymbol{f}}{\partial \boldsymbol{x}} (\boldsymbol{x})
\rVert_F
\;
d\boldsymbol{x}
=
\Vol(S) 
\cdot  
\underset{\boldsymbol{x} \sim S}{\mathbb{E}} [
\lVert
\tfrac{\partial \boldsymbol{f}}{\partial \boldsymbol{x}} (\boldsymbol{x})
\rVert_F
]
\end{equation}
where $\boldsymbol{x} \sim S$ is uniformly sampled
is the support set of the Jacobian norm over the open set $\mathcal{O}$. 
Based on the decomposition, the inter-class distance maximization increases the volume $\Vol(S)$ of the set $S$ and/or the expected Jacobian norm over $S$.
\end{proof}

\section{Proofs to the Method}
\label{sec:proof_method}

\begin{namedtheorem}[Propositions 6 and 7]
With the optimal prototypes $\{\boldsymbol{w}_k\}_{k=1}^K$ for a single sample reprsentation $\boldsymbol{f}(\boldsymbol{x})$ paired with label $y$, we have the collapse $\boldsymbol{w}_k {=} {-}\boldsymbol{w}_y$ for all $k {\neq} y$ if $m_n {=} 0$, while $\measuredangle(\boldsymbol{w}_k, {-}\boldsymbol{w}_y) {\geq} m_n $ with $k {\neq} y$ if $0 {<} m_n {<} \pi/2$.
\end{namedtheorem}

\begin{proof}[Proof of Propositions 6 and 7]
For the optimal prototypes $\{\boldsymbol{w}_k\}_{k=1}^K$, we have $s_y = \max_{\boldsymbol{w}_y} s_y$ and $s_k = -1$. Regardless of whether $m_n >0$ or not, we have $\boldsymbol{f}(\boldsymbol{x}) = \boldsymbol{w}_y$. For the negative pair, if $m_n=0$, then $\boldsymbol{f}(\boldsymbol{x}) = \boldsymbol{w}_k$, and hence $-\boldsymbol{w}_y = \boldsymbol{w}_k$ for all $k \neq y$. If $m_n > 0$, then $s_k = -1$ when the angle between $\boldsymbol{w}_k$ and $\boldsymbol{f}(\boldsymbol{x})$ is $\pi - m_n$, implying that $\measuredangle(\boldsymbol{w}_k, {-}\boldsymbol{w}_y) = \measuredangle(\boldsymbol{w}_k, {-}\boldsymbol{f}(\boldsymbol{x})) =  m_n $, finishing the proof.
\end{proof}

\begin{proof}[Proof of Proposition \ref{prop:movr_sce}]
Observe that
\begin{equation}
\frac{\partial \mathcal{L}_{\text{SCE}} }{\partial \theta} 
= - c (\sum_{j \neq y} e^{s_j - s_y}) \frac{\partial s_y}{\partial \theta}
+ c\sum_{k \neq y} e^{s_k-s_y} \frac{\partial s_k}{\partial \theta}
\end{equation}
where $\mathcal{L}_{\text{SCE}}=\log(1+ \sum_{k \neq y}e^{s_k-s_y})$ and $c = (1+ \sum_{j \neq y} e^{s_j - s_y})^{-1}$.
On the other hand,
\begin{equation}
\frac{L_{\text{m-OvR}}}{\partial \theta} 
= - [1 + e^{s_y}]^{-1} \frac{\partial s_y}{\partial \theta} 
+ \sum_{k \neq y} [1 + e^{s_k}]^{-1} \frac{\partial s_k}{\partial \theta}.
\end{equation}
Hence, the inter-class gradient $\frac{\partial s^{\text{SCE}}_k}{\partial \theta}$ for SCE is
\begin{equation}
\frac{\partial s^{\text{SCE}}_k}{\partial \theta}
= c e^{s_k-s_y} \frac{\partial s_k}{\partial \theta},
\end{equation}
while the inter-class gradient for $\frac{\partial s^{\text{m-OvR}}_k}{\partial \theta}$ is
\begin{equation}
\frac{\partial s^{\text{m-OvR}}_k}{\partial \theta}
= [1 + e^{s_k}]^{-1} \frac{\partial s_k}{\partial \theta}.
\end{equation}
To prove our claim, it suffices to show that $\frac{\partial s^{\text{m-OvR}}_k}{\partial \theta} > \frac{\partial s^{\text{SCE}}_k}{\partial \theta}$. To this end, observe that
\begin{equation}
\frac{e^{s_k - s_y}}{1 + \sum_{j \neq y} e^{s_j - s_y}}
<
\frac{1}{1 + e^{-s_k}},
\end{equation}
which is equivalent to 
\begin{equation}
e^{s_k} + 1 < \sum_j e^{s_j} = e^{s_k} + \sum_{j \neq k, y} e^{s_j} + e^{s_y},
\end{equation}
which holds if $s^{s_y} > 0$ due to $ \sum_{j \neq k, y} e^{s_j} \geq 0$. This completes our proof.
\end{proof}

\section{Additional Empirical Results}
\label{appendix:jacobian_softmax}

The results given in \ref{fig:iter_vs_metric_softmax} show that the Jacobian norm trend that we observed in the main sections holds the same way for the softmax cross entropy models.

\begin{figure*}[b]
\begin{center}
\includegraphics[width=.75\linewidth]{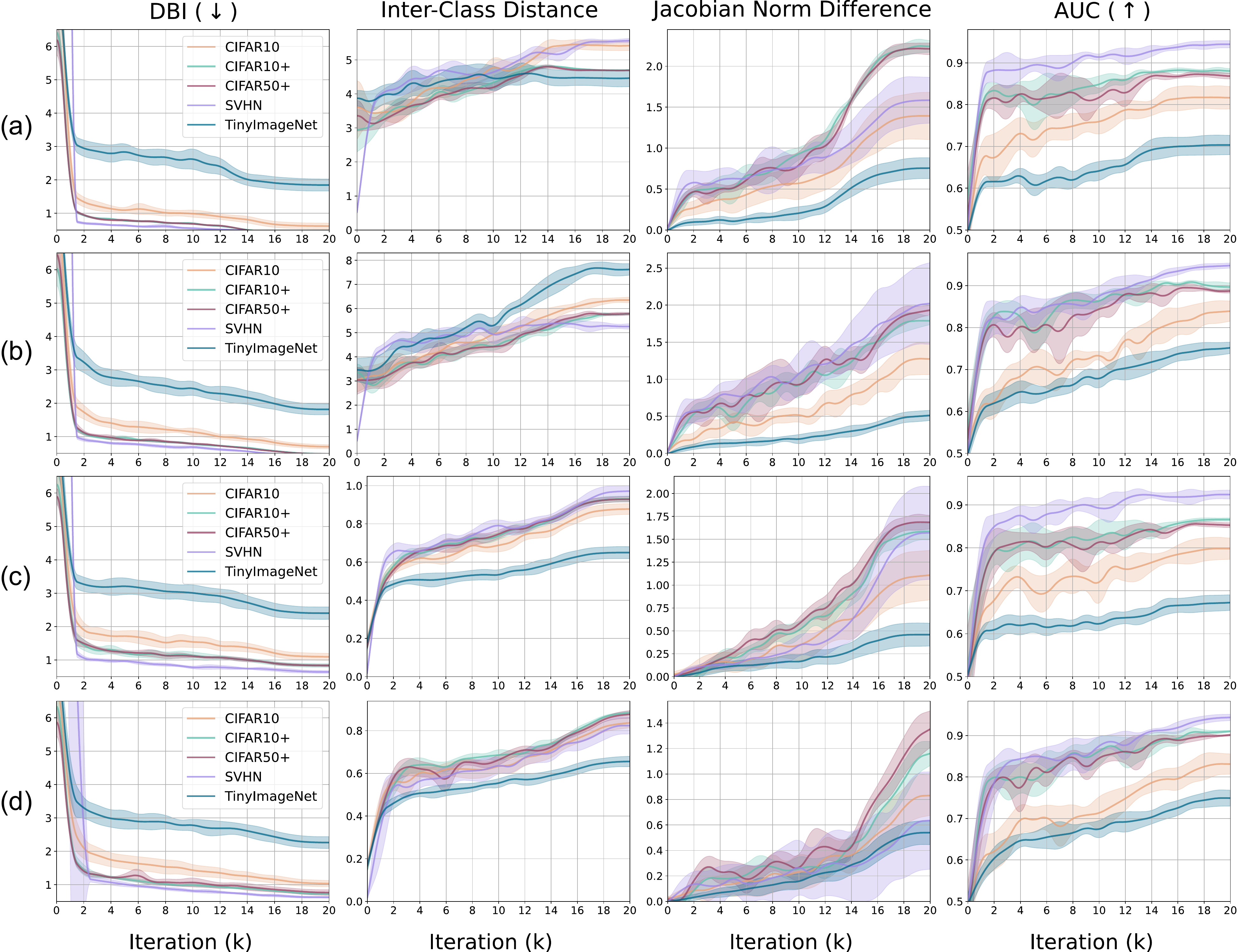}
\end{center}
\caption{
Several metrics measured while different discriminative models are begin trained. (a) SCE without data augmentation, (b) SCE with data augmentation, (c) SCE with normalized embedding but without data augmentation, (d) SCE with normalized embedding and data augmentation.
}
\label{fig:iter_vs_metric_softmax}
\end{figure*}